\documentclass{article}

\usepackage[T1]{fontenc}
\usepackage[htt]{hyphenat}

\usepackage{microtype}
\usepackage{graphicx}
\usepackage{subfigure}
\usepackage{booktabs} 
\usepackage{natbib}
\usepackage{algorithm}
\usepackage{algorithmic}
\usepackage{siunitx}
\usepackage{subfigure}
\usepackage[subfigure]{tocloft}

\let\oldaddcontentsline\addcontentsline

\newcommand{\starttocentries}{\let\addcontentsline\oldaddcontentsline}

\usepackage{amsmath}
\usepackage{amssymb}
\usepackage{graphicx}
\usepackage{mathtools}
\usepackage{amsfonts}
\usepackage{amsthm,bm}
\usepackage{color}
\usepackage{enumitem}
\usepackage{dsfont}
\usepackage{pgfplots}
\pgfplotsset{width=10cm,compat=1.9}
 \usepackage{pifont}
\usepackage{thmtools} 
\usepackage{thm-restate}
\usepackage{sidecap}

\newcommand{\R}{\mathbb{R}}

\renewcommand{\phi}{\varphi}

\graphicspath {{figures/}}

\usepackage{hyperref}
\usepackage{caption}
\usepackage[super]{nth}
\newtheorem{lemma}{Lemma}
\newtheorem{definition}{Definition}
\newtheorem{theorem}{Theorem}

\newtheorem{proposition}{Proposition}

\newtheorem{corollary}{Corollary}

\mathtoolsset{showonlyrefs}

\usepackage{breakcites}
\usepackage{authblk}
\usepackage[margin=1.39in]{geometry}
\usepackage{enumitem}

\usepackage{hyperref}

\hypersetup{
     colorlinks = true,
     linkcolor = brown,
     anchorcolor = blue,
     citecolor = teal,
     filecolor = blue,
     urlcolor = black
     }

\title{Auditing Differential Privacy in High Dimensions with the Kernel Quantum Rényi Divergence}

\author[a]{Carles Domingo-Enrich}
\author[b]{Youssef Mroueh}
\affil[a]{Courant Institute of Mathematical Sciences, New York University}
\affil[b]{IBM Research AI}

\begin{document}

\maketitle


\begin{abstract}%
  Differential privacy (DP) is the de facto standard for private data release and private machine learning.  Auditing black-box DP algorithms and mechanisms to certify whether they  satisfy a certain DP guarantee is challenging, especially in high dimension. We propose relaxations of differential privacy based on new divergences on probability distributions: the kernel Rényi divergence and its regularized version. We show that the regularized kernel Rényi divergence can be estimated from samples even in high dimensions, giving rise to auditing procedures for $\varepsilon$-DP, $(\varepsilon,\delta)$-DP and $(\alpha,\varepsilon)$-Rényi DP.
\end{abstract}

\addtocontents{toc}{\protect\setcounter{tocdepth}{0}}

\section{Introduction}
Differential privacy or DP \citep{dwork2006calibrating} is the main paradigm adopted across many industries and organisations for private data release and for performing private machine learning. For instance, the US Census bureau \citep{UScensus} adopted differential privacy mechanisms to release sensitive records from the US census. Deep learning with differential privacy 
has become popular with the availability of many open-source libraries implementing variants of Differentially Private-SGD (DP-SGD) \citep{DP-SGD} such as TensorFlow Privacy \citep{TFPrivacy} and Opacus in PyTorch \citep{opacus}. 

Differential privacy provides formal worst case guarantees of privacy for any two adjacent datasets against any potential attacker. The privacy budget is controlled by  parameters $\varepsilon$ and $\delta$. For privacy preservation, $\varepsilon$ is chosen to be small but this comes at the detriment of the accuracy or the utility of the private data or model at hand. Typical values in private deep learning are $\varepsilon=2$ or $3$ \citep{DP-SGD}. While these values ensure an acceptable level of accuracy, they  raise a concern regarding the effective privacy of such values chosen in practice with respect to formal guarantees. Moreover,  these guarantees are derived from upper bounds on the privacy budget that are often not tight.   

In order to probe and audit the privacy of data and models, many privacy attacks have been proposed to assess the data leakage in machine learning such as black-box and white-box membership attacks  \citep{homer2008resolving, shokri2017membership, yeom2018privacy},
as well as attacks on differentially private machine learning \citep{Jayaraman2019,sablayrolles2019white,Rahimian2021DifferentialPD,chen2020gan}. While privacy attacks provide a qualitative assessment of DP, they don't provide quantitative privacy guarantees nor detect precisely violations of differential privacy in terms of the desired privacy budget $\varepsilon$. \cite{Ding2018} designed a hypothesis testing pipeline to audit violations of DP, and  \cite{JagielskiUO20} audits violations of privacy of DP-SGD using poisoning attacks and a carefully crafted testing that provides a lower bound on the privacy budget. 

In this work, 
we put forward a way to audit DP by introducing a new notion of differential privacy, the \emph{kernel Rényi differential privacy (KRDP)}, which builds on a new divergence on distributions that we propose in this work: the \emph{kernel Rényi divergence (KRD)}. The KRD between two distributions is defined as the quantum Rényi divergence between their covariance operators in a reproducing kernel Hilbert space, in analogy with the kernel relative entropy defined in the seminal paper of \cite{bach2022information}. Armed with this divergence, our definition of KRDP mirrors Rényi differential privacy \citep{mironov2017renyi} and provides a rigorous testing framework for violations of differential privacy. While previous works focused on univariate testing of DP \citep{Ding2018,JagielskiUO20}, our proposed auditing method does not suffer from curse of dimension and has parametric rates in high dimensions.

Namely, the main contributions of this work are: 
\begin{itemize}[leftmargin=20pt,topsep=0pt,itemsep=0pt]
    \item We propose new notions of differential privacy: the kernel Rényi differential privacy (KRDP) and the regularized kernel Rényi differential privacy (RKRDP) (\autoref{sec:framework}), which are based on the (regularized) kernel Rényi divergence (RKRD). We study their convexity, postprocessing and composition properties (\autoref{sec:properties}).
    \item We show that the RKRD can be estimated from samples (\autoref{sec:estimating}), which yields methods to test whether a black-box mechanism with high-dimensional output satisfies $\varepsilon$-DP, $(\varepsilon,\delta)$-DP, $(\alpha,\varepsilon)$-Rényi DP and RKRDP (\autoref{sec:auditing}). We validate our method in \autoref{sec:exp}.
\end{itemize}


\section{Framework} \label{sec:framework}
Let $\mathcal{X} \subseteq \R^d$, 
and let $k : \mathcal{X} \times \mathcal{X} \to \R$ be a continuous positive definite kernel on $\mathcal{X}$, such that $k(x,x) = 1$ for all $x \in \mathcal{X}$. Implicitly, we make this assumption on all kernels throughout the paper. Let $\mathcal{H}$ be a reproducing kernel Hilbert space of functions $f: \mathcal{X} \to \mathbb{R}$ corresponding to the kernel $k$. Let $\phi : \mathcal{X} \to \mathcal{H}$ be a map such that
\begin{align}
    \forall f \in \mathcal{H}, \forall x,y \in \mathcal{X}, \quad f(x) = \langle f, \phi(x) \rangle, \quad k(\cdot,x) = \phi(x), \quad k(x,y) = \langle \phi(x), \phi(y) \rangle.
\end{align}

\textbf{Positive Hermitian operators.} A linear map $A$ from a Hilbert space $\mathcal{H}$ to itself is positive semidefinite or positive if for any $h \in \mathcal{H}$, $\langle h, A h \rangle \geq 0$, and positive definite or strictly positive if $\langle h, A h \rangle > 0$ for any $h \neq 0$ \footnote{More generally, in a $C^{*}$-algebra $\mathcal{A}$, an element is positive if it is of the form $a a^{*}$ for some $a \in \mathcal{A}$.}. We write $A \geq 0$ to denote that $A$ is positive, and $A > 0$ to denote that it is strictly positive. 
$A$ is Hermitian or self-adjoint if for any $h, h' \in \mathcal{H}$, $\langle h, A h' \rangle = \langle A h, h' \rangle$. If $A,B$ are positive Hermitian operators, we say that $A$ dominates $B$, or $A \gg B$, if the kernel of $A$ is contained in the kernel of $B$.

\textbf{The covariance operator.} We consider the function $x \mapsto \phi(x) \phi(x)^{*}$, which maps $\mathcal{X}$ to the set of linear maps from $\mathcal{H}$ to $\mathcal{H}$. Given a probability distribution $p$ on $\mathcal{X}$, we define the covariance operator $\Sigma_p : \mathcal{H} \to \mathcal{H}$ as
\begin{align}
    \Sigma_p = \int_{\mathcal{X}} \phi(x) \phi(x)^{*} dp(x).
\end{align}
By definition, for any $f, g \in \mathcal{H}$, $\langle g, \Sigma_p f \rangle = \mathbb{E}_{X \sim p} [\langle g, \phi(X) \rangle \langle \phi(X), f \rangle] = \mathbb{E}_{X \sim p} [g(X) f(X)]$. Also, note that the assumption $k(x,x) = 1$ implies that $\mathrm{tr}[\Sigma_p] = \int_{\mathcal{X}} \langle \phi(x), \phi(x) \rangle dp(x) = 1$. The covariance operator has an analog in quantum mechanics, known as the density operator, which is used to represent mixed quantum states. In the bra-ket notation, we can write the covariance operator as $\Sigma_p = \int_{\mathcal{X}} | \phi(x) \rangle \langle \phi(x) | dp(x)$. 

As noted by \cite{bach2022information} (Sec. 2.1), under our assumptions the map $p \to \Sigma_p$ is injective when $k$ is a universal kernel on $\mathcal{X}$, which by definition means that the RKHS is dense in the set of continuous functions equipped with the uniform norm. Most of the common kernels are universal when $\mathcal{X}$, e.g. the radial basis function (RBF) kernel $k(x,y) = \exp(-\frac{\|x-y\|^2}{\lambda^2})$. Hence, one can measure the difference between two distributions $p$ and $q$ by looking how far apart the operators $\Sigma_p$ and $\Sigma_q$ are in a certain metric. \cite{bach2022information} follows this recipe to define the \textit{kernel Kullback-Leibler (KL) divergence} between $p$ and $q$: the \textit{quantum KL divergence} or \textit{quantum relative entropy} between two positive Hermitian operators with finite trace is defined as $D(A||B) = \mathrm{tr}[A (\log A -\log B)]$, and then the kernel Kullback-Leibler (KL) divergence between $p$ and $q$ is simply $D(\Sigma_p||\Sigma_q) = \mathrm{tr}[\Sigma_p (\log \Sigma_p -\log \Sigma_q)]$.

\textbf{The (regularized) quantum Rényi divergence.}
We use a similar construction, but we start instead from the (regularized) quantum Rényi divergence. 
For $\alpha \in (0,1) \cup (1,\infty)$, the \textit{order-$\alpha$ quantum Rényi divergence} was introduced by the seminal work of \cite{mullerlennert2013onquantum}. For $A,B$ positive Hermitian operators such that $A \geq 0$ and $\mathrm{tr}(A) < +\infty$, it is defined as:
\begin{align} \label{eq:quantum_renyi}
    D_{\alpha}(A||B) := 
    \frac{1}{\alpha-1} \log \left( \frac{1}{\mathrm{tr}[A]} \mathrm{tr}\left[\left( B^{\frac{1-\alpha}{2\alpha}} A B^{\frac{1-\alpha}{2\alpha}} \right)^\alpha \right] \right) 
\end{align}
Note that there are instances in which $D_{\alpha}(A||B)$ takes value $+\infty$, e.g. when $\alpha > 1$ and the kernel of $B$ is not contained in the kernel of $A$. In the limit $\alpha \to 1$, the quantum KL divergence is recovered $\lim_{\alpha \to 1} D_{\alpha}(A||B) = D(A||B)$.

In a similar spirit, for $\lambda \in (0,+\infty)$ and $\alpha \in (0,1) \cup (1,\infty)$, we define the \textit{regularized quantum Rényi divergence} as
\begin{align} \label{eq:reg_quantum_renyi}
    D_{\alpha,\lambda}(A||B) := D_{\alpha}(A||B+\lambda \mathrm{Id}) = \frac{1}{\alpha-1} \log \left( \frac{1}{\mathrm{tr}[A]} \mathrm{tr}\left[\left( (B+\lambda \mathrm{Id})^{\frac{1-\alpha}{2\alpha}} A (B+\lambda \mathrm{Id})^{\frac{1-\alpha}{2\alpha}} \right)^\alpha \right] \right)
\end{align}
Namely, $D_{\alpha,\lambda}(A||B) = D_{\alpha}(A||B+\lambda \mathrm{Id})$, and when we set $\lambda = 0$ we recover the quantum Rényi divergence. \cite{mullerlennert2013onquantum} showed that the quantum Rényi divergence satisfies six desirable properties that a previously proposed Rényi divergence on operators did not fulfill. 
We generalize these properties to the regularized quantum Rényi divergence, as we need several of them in our analysis:

\begin{proposition}[Properties of regularized quantum Rényi divergence] \label{prop:prop_regularized}
    For $\alpha \in [1/2,1)\cup(1,+\infty)$, the regularized quantum Rényi divergence fulfills a set of nice properties: 
\begin{enumerate}[leftmargin=25pt,topsep=-5pt,itemsep=-5pt]
    \item Continuity: $D_{\alpha, \lambda}(A||B)$ is continuous in $A, B \geq 0$, wherever $A \neq 0$ and $B + \lambda \mathrm{Id} \gg A$.
    \item Unitary invariance: $D_{\alpha, \lambda}(A||B) = D_{\alpha, \lambda}(UA U^{*}||U B U^{*})$ for any unitary $U$.
    \item Normalization: $D_{\alpha,\lambda}(1||\frac{1}{2}) = -\log(\frac{1}{2} + \lambda)$.
    \item Order: If $A \geq B$ (i.e. if $A - B \geq 0$ is positive semi-definite), then $D_{\alpha,0}(A||B) = D_{\alpha}(A||B) \geq 0$. 
    If $A \leq B$, then $D_{\alpha,\lambda}(A||B) \leq 0$.
    \item Additivity: For all $A,B,C,D \geq 0$ such that $B \gg A$ and $D \gg C$, we have that $D_{\alpha,\lambda}(A \otimes C||B \otimes D) \geq D_{\alpha,\sqrt{\lambda}}(A||B) + D_{\alpha,\sqrt{\lambda}}(C||D)$ for any $\lambda \geq 0$. If moreover $B, D \leq \mathrm{Id}$, then $D_{\alpha,\lambda}(A \otimes C||B \otimes D) \leq D_{\alpha,\lambda/3}(A||B) + D_{\alpha,\lambda/3}(C||D)$ for any $3 \geq \lambda \geq 0$. 
    \item Mean: Let $g(x) = \exp((\alpha-1) x)$. For all $A,B,C,D \geq 0$ with $A \neq 0$, $C \neq 0$, $B \gg A$, $D \gg C$, $\mathrm{tr}[A+C] \leq 1$ and $\mathrm{tr}[B+D] \leq 1$,
    \begin{align}
        D_{\alpha,\lambda}(A \oplus C||B \oplus D) = g^{-1} \bigg( \frac{\mathrm{tr}[A]}{\mathrm{tr}[A+C]} g(D_{\alpha,\lambda}(A||B)) + \frac{\mathrm{tr}[C]}{\mathrm{tr}[A+C]} g(D_{\alpha,\lambda}(C||D)) \bigg)
    \end{align}
\end{enumerate}
\end{proposition}

\textbf{The (regularized) kernel Rényi divergence.} Then, we define the \textit{(regularized) kernel Rényi divergence (RKRD)} between two probability measures $p,q$ on $\mathcal{X}$ as the (regularized) quantum Rényi divergence between their covariance operators, that is: 
\begin{align} \label{eq:kernel_renyi}
    D_{\alpha,\lambda}(\Sigma_p||\Sigma_q) := 
    \frac{1}{\alpha-1} \log \left( \mathrm{tr}\left[\left( (\Sigma_q+\lambda \mathrm{Id})^{\frac{1-\alpha}{2\alpha}} \Sigma_p (\Sigma_q+\lambda \mathrm{Id})^{\frac{1-\alpha}{2\alpha}} \right)^\alpha \right] \right) 
\end{align}
Here, we used that $\mathrm{tr}[\Sigma_p] = 1$. In the setting $\lambda = 0$, we refer to $D_{\alpha,0}(\Sigma_p||\Sigma_q) = D_{\alpha}(\Sigma_p||\Sigma_q)$ simply as the \textit{kernel Rényi divergence (KRD)}. The kernel Rényi divergence is to be contrasted with the classical Rényi divergence, introduced by \cite{renyi1961onmeasures}. 
The \textit{Rényi divergence (RD)} is defined for two probability measures $p,q$ on $\mathcal{X}$ as $D_{\alpha}(p||q) = \frac{1}{\alpha-1} \log(\int_{\mathcal{X}} (\frac{dp}{dq}(x))^{\alpha} \, dq(x))$.
When the probability measures admit densities, this simplifies to $D_{\alpha}(p||q) = \frac{1}{\alpha-1} \log( \int_{\mathcal{X}} p(x)^{\alpha} q(x)^{1-\alpha} \, dx )$.
In the limit $\alpha \to 1$, the KL divergence is recovered: $D_1(p||q) = \int_{\mathcal{X}} \log(\frac{dp}{dq}(x)) \, dp(x)$. In the limit $\alpha \to +\infty$, we obtain $D_{\infty}(p||q) = \sup_{x \in \mathrm{supp}(q)} \log(\frac{dp}{dq}(x))$. The following proposition, proven in \autoref{sec:proofs_framework}, 
shows that the RKRD decreases with $\lambda$ at a fixed $\alpha$, increases with $\alpha$ at a fixed $\lambda$, and is upper-bounded by the RD.
\begin{proposition} \label{prop:comparison_renyi}
For any two probability measures $p$, $q$ on $\mathcal{X}$, kernel $k$, $\alpha \in [1/2,1)\cup(1,+\infty)$ and $+\infty \geq \lambda_1 \geq \lambda_2 \geq 0$, we have that $D_{\alpha,\lambda_1}(\Sigma_p||\Sigma_q) \leq D_{\alpha,\lambda_2}(\Sigma_p||\Sigma_q) \leq D_{\alpha}(p||q)$. Also, for any kernel $k$, $\lambda \in [0,+\infty)$, and $\alpha_1, \alpha_2 \in (0,1)\cup(1,+\infty)$ with $\alpha_1 \geq \alpha_2$, we have that $D_{\alpha_2,\lambda}(\Sigma_p||\Sigma_q) \leq D_{\alpha_1,\lambda}(\Sigma_p||\Sigma_q)$.
\end{proposition}
Proceeding along the lines of Sec. 4.1 of \cite{bach2022information}, which lower bounds the kernel KL divergence, the next proposition shows a complementary lower bound of the kernel Rényi divergence by the Rényi divergence between a pair of related distributions. The proof is in \autoref{subsec:proof_lower_bound}. 
\begin{proposition} \label{prop:lower_bound_renyi}
Suppose that $\mathcal{X}$ is compact and let $\tau$ be the uniform measure over $\mathcal{X}$. Let $\Sigma$ be the covariance operator of $\tau$, i.e. $\Sigma = \int_{\mathcal{X}} \phi(x) \phi(x) \, d\tau(x)$. Define $h : \mathcal{X} \times \mathcal{X} \to \R$ as $h(x,y) = \langle \phi(x), \Sigma^{-1/2} \phi(y) \rangle$. Given probability measures $p, q$, we define the probability densities $\tilde{p}(y) = \int_{\mathcal{X}} h(x,y) \, dp(x)$, $\tilde{q}(y) = \int_{\mathcal{X}} h(x,y) \, dq(x)$. We have that $D_{\alpha}(\tilde{p}||\tilde{q}) \leq D_{\alpha}(\Sigma_p||\Sigma_q)$.
\end{proposition}

\textbf{Rényi differential privacy.} \cite{mironov2017renyi} used this last observation as the inspiration to define a relaxation of differential privacy based on the Rényi divergence. Consider a set $\mathcal{D}$ with a relation $\sim$ that describes whether two elements $D, D' \in \mathcal{D}$ are adjacent. In the prototypical differential privacy setting, $\mathcal{D}$ is a set of databases, and two databases $D, D' \in \mathcal{D}$ are adjacent if they differ in one entry.
\begin{definition}[$(\alpha,\varepsilon)$-Rényi differential privacy (RDP), \cite{mironov2017renyi}, Def. 4]
A randomized mechanism $f: \mathcal{D} \to \mathcal{R}$ is said to have $\varepsilon$-Rényi differential privacy of order $\alpha$, or $(\alpha,\varepsilon)$-RDP for short, if for any adjacent $D, D' \in \mathcal{D}$ it holds that $D_{\alpha}(\mathcal{L}(f(D))||\mathcal{L}(f(D'))) \leq \varepsilon$, where $\mathcal{L}(f(D))$ denotes the law (distribution) of the random variable $f(D)$ and $D_{\alpha}$ denotes the (classical) Rényi divergence. 
\end{definition}
In this definition, the output space $\mathcal{R}$ is usually taken to be $\R^{d}$, where $d$ can be high or low-dimensional. 
In the limit $\alpha \to +\infty$, the defining condition of $(\alpha,\varepsilon)$-RDP becomes $\sup_{x \in \mathrm{supp}(q)} \log(\frac{dp}{dq}(x)) < \varepsilon$ for $p = \mathcal{L}(f(D))$ and $q = \mathcal{L}(f(D'))$. This is equivalent to the original definition of \textit{$\varepsilon$-differential privacy} (DP, \cite{dwork2006calibrating}), which is typically stated as $\mathrm{Pr}[f(D) \in \mathcal{S}] \leq e^{-\varepsilon} \mathrm{Pr}[f(D') \in \mathcal{S}]$ for any measurable $\mathcal{S}$ and any $D, D' \in \mathcal{D}$ adjacent. Given $\delta > 0$, another well-known relaxation of $\varepsilon$-DP which we also build on is \textit{($\varepsilon,\delta$)-differential privacy} \citep{dwork2006ourdata}, which imposes the condition $\mathrm{Pr}[f(D) \in \mathcal{S}] \leq e^{-\varepsilon} \mathrm{Pr}[f(D') \in \mathcal{S}] + \delta$, and is hence less stringent than $\varepsilon$-DP.

\textbf{(Regularized) kernel Rényi differential privacy.} We define the (regularized) kernel Rényi differential privacy in analogy with Rényi differential privacy:
\begin{definition}[$(k,\alpha,\lambda,\varepsilon)$-regularized kernel Rényi differential privacy (RKRDP)]
A randomized mechanism $f: \mathcal{D} \to \mathcal{R}$ is said to have $\varepsilon$-kernel Rényi differential privacy of order $\alpha \in (0,1) \cup (1,+\infty)$ and regularization $\lambda \in [0,+\infty)$ with respect to a kernel $k$, or $(k,\alpha,\lambda,\varepsilon)$-RKRDP for short, if for any adjacent $D, D' \in \mathcal{D}$ it holds that $D_{\alpha,\lambda}(\Sigma_{\mathcal{L}(f(D))}||\Sigma_{\mathcal{L}(f(D'))}) \leq \varepsilon$, where the covariance matrices $\Sigma_{\mathcal{L}(f(D))}, \Sigma_{\mathcal{L}(f(D'))}$ for the distributions of $f(D)$ and $f(D')$ are with respect to feature map $\phi$ of the kernel $k$.
\end{definition}
When $\lambda = 0$, we talk about $(k,\alpha,\varepsilon)$-kernel Rényi differential privacy or $(k,\alpha,\varepsilon)$-KRDP.
The following proposition that shows implications between RKRDP conditions with different values of $\alpha, \lambda$, and different kernels, and links RDP with RKRDP. The proof is in \autoref{sec:proofs_properties}.

\begin{proposition}
[Implications regarding RKRDP] \label{prop:krdp_implications}
The following implications 
hold:
\begin{enumerate}[leftmargin=25pt,nolistsep]
    \item For any kernel $k$, $\alpha \in [1/2,1)\cup(1,+\infty)$ and $\varepsilon > 0$, if $\lambda_1 \geq \lambda_2 \geq 0$, $(k,\alpha,\lambda_2,\varepsilon)$-RKRDP implies $(k,\alpha,\lambda_1,\varepsilon)$-RKRDP.
    \item For any kernel $k$, $\lambda \in [0,+\infty)$ and $\varepsilon > 0$, $(k,\alpha_1,\lambda,\varepsilon)$-RKRDP implies $(k,\alpha_2,\lambda,\varepsilon)$-RKRDP when $\alpha_1 \geq \alpha_2$, where $\alpha_1, \alpha_2 \in (0,1)\cup(1,+\infty)$.
    \item For any kernel $k$ and $\alpha \in [1/2,1)\cup(1,+\infty)$, $(\alpha,\varepsilon)$-RDP implies $(k,\alpha,\lambda,\varepsilon)$-RKRDP for all $\lambda \in [0,+\infty)$.
    \item If the kernel $k$ is of the form $k(x,x') = \prod_{j=1}^{J} k_j(x,x')$, where $k_j$ are kernels,
    then $(k,\alpha,\lambda,\varepsilon)$-RKRDP implies $(k_j,\alpha,\lambda^{1/J},\varepsilon)$-RKRDP, for any $j \in \{1,\dots,J\}$ and any $\alpha \in [1/2,1)\cup(1,+\infty)$.
\end{enumerate}
\end{proposition}
The first three implications in \autoref{prop:krdp_implications} follow from the bounds on the regularized kernel Rényi divergence shown in \autoref{prop:comparison_renyi}. In short, \autoref{prop:krdp_implications} states that the $(k,\alpha,\lambda,\varepsilon)$-RKRDP condition is stronger when $\lambda$ is small and $\alpha$ is large, and that $(\alpha,\varepsilon)$-RDP implies $(k,\alpha,\lambda,\varepsilon)$-RKRDP regardless of the choice of $k$ and $\lambda$. That is, $(k,\alpha,\lambda,\varepsilon)$-RKRDP is a relaxation of $(\alpha,\varepsilon)$-RDP, and as such, is also a relaxation of $\varepsilon$-DP.


\section{Properties of the (regularized) kernel Rényi differential privacy} \label{sec:properties}

The success of differential privacy can be traced to its desirable theoretical properties that make it amenable for practical use. \cite{kifer2012arigorous} posits that a good privacy notion should have three properties: convexity, postprocessing invariance and graceful composition. These apply to $\varepsilon$-DP and to $(\alpha,\varepsilon)$-RDP 
\citep{mironov2017renyi}. \cite{chaudhuri2019capacity} show that their notion of capacity bounded differential privacy satisfies all three properties. We now show that many of these properties continue to hold for the (regularized) kernel Rényi differential privacy.

We begin with the convexity property, which guarantees that mixtures of private mechanisms preserve the privacy notion. The proof is in \autoref{sec:proofs_properties}. 
\begin{proposition}[Convexity property] \label{prop:convexity}
Let $\alpha \in (1,+\infty)$ and $\beta \in [0,1]$ arbitrary. If $f, g :\mathcal{D} \to \mathcal{R}$ are randomized maps that satisfy $(k,\alpha,\lambda,\varepsilon)$-RKRDP, then the randomized map $h :\mathcal{D} \to \mathcal{R}$ defined as the mixture of $f$ with probability $\beta$ and $g$ with probability $1-\beta$ satisfies $(k,\alpha,\lambda,\varepsilon)$-RKRDP.
\end{proposition}

\subsection{Postprocessing properties}

Usual privacy notions satisfy that applying any function to the output of a private mechanism does not degrade the privacy guarantee. For example, for $(\alpha,\varepsilon)$-RDP, this is an immediate consequence of the data processing inequality for the Rényi divergence, which states if $X,Y$ are random variables on $\mathcal{R}$ with distributions $\mathcal{L}(X)$, $\mathcal{L}(X)$, and $g : \mathcal{R} \to \mathcal{R}'$, then $D_{\alpha}(\mathcal{L}(g(X))||\mathcal{L}(g(Y))) \leq D_{\alpha}(\mathcal{L}(X)||\mathcal{L}(Y))$. The following theorem shows that the kernel Rényi divergence (i.e. the RKRD with $\lambda = 0$) fulfills a similar inequality under a certain condition on the kernels $k, k'$ and the map $g$, and implies that KRDP is preserved.

\begin{theorem}[Postprocessing property for KRDP under deterministic $g$] \label{thm:postprocessing_deterministic}
Suppose that $\mathcal{H}, \mathcal{H}'$ are RKHS on $\mathcal{R}, \mathcal{R}'$ respectively, with feature maps $\phi : \mathcal{R} \to \mathcal{H}$, $\psi : \mathcal{R}' \to \mathcal{H}'$, and kernel functions $k,k'$ respectively, such that 
there exists a kernel function $\tilde{k}$ such that
\begin{align} \label{eq:k_k'_k''_condition}
    k(x,x') = k'(g(x),g(x')) \tilde{k}(x,x'), \quad \forall x, x' \in \mathcal{R}.
\end{align}
If $f : \mathcal{D} \to \mathcal{R}$ is a randomized map and $g : \mathcal{R} \to \mathcal{R}'$ is a deterministic map, then for any pair $D,D' \in \mathcal{D}$, we have that
\begin{align} \label{eq:contraction_postprocessing}
    D_{\alpha}(\Sigma_{\mathcal{L}(g(f(D)))}||\Sigma_{\mathcal{L}(g(f(D')))}) \leq D_{\alpha}(\Sigma_{\mathcal{L}(f(D))}||\Sigma_{\mathcal{L}(f(D'))}).
\end{align}
Consequently, if $f$ satisfies $(k,\alpha,\varepsilon)$-KRDP then $g \circ f$ satisfies $(k',\alpha,\varepsilon)$-KRDP.
\end{theorem}

We can also state a more general version of the postprocessing property in which the map $g$ can be randomized.

\begin{theorem}[Postprocessing property for randomized $g$] \label{thm:postprocessing_randomized} 
Suppose that $\mathcal{H}, \mathcal{H}'$ are RKHS on $\mathcal{R}, \mathcal{R}'$ respectively, with feature maps $\phi : \mathcal{R} \to \mathcal{H}$, $\psi : \mathcal{R}' \to \mathcal{H}'$, and kernel functions $k,k'$ 
such there exists a family of kernels $(\tilde{k}_g)_{g}$ with $k_g(x,x) = 1$ such that 
\begin{align}
k(x,x') = \mathbb{E}_{g}[\tilde{k}_g(x,x') k'(g(x),g(x'))].
\end{align}
If $f : \mathcal{D} \to \mathcal{R}$ and $g : \mathcal{R} \to \mathcal{R}'$ are randomized mappings, then for any pair $D,D' \in \mathcal{D}$, we have that
\begin{align} \label{eq:contraction_postprocessing_randomized}
    D_{\alpha}(\Sigma_{\mathcal{L}(g(f(D)))}||\Sigma_{\mathcal{L}(g(f(D')))}) \leq D_{\alpha}(\Sigma_{\mathcal{L}(f(D))}||\Sigma_{\mathcal{L}(f(D'))}).
\end{align}
Consequently, if $f$ satisfies $(k,\alpha,\varepsilon)$-KRDP then $g \circ f$ satisfies $(k',\alpha,\varepsilon)$-KRDP.
\end{theorem}

To show \eqref{eq:contraction_postprocessing} and \eqref{eq:contraction_postprocessing_randomized}, we use a version of the data processing inequality for the quantum Rényi divergence \citep{frank2013monotonicity,beigi2013sandwiched} which works under a class of operators known as completely positive trace-preserving (CPTP) maps. The main challenge is to show the existence of the appropriate CPTP maps, and for this we generalize to the infinite-dimensional case some results from the quantum information literature \citep{chefles2004ontheexistence} which may be of independent interest. The framework and proofs are in \autoref{subsec:proofs_postprocessing}. We state the postprocessing properties for $\lambda = 0$ only because the data processing inequality does not seem to extend to the RKRD with $\lambda > 0$ (left as an open problem).

\subsection{Composition properties}
Generally speaking, the composition properties of a privacy notion refer to the privacy guarantees of the joint application of private mechanisms. We provide a first result for RKRDP in which the mechanisms are run independently and in parallel. In this setting, the guarantees for $\varepsilon$-DP and $(\alpha,\varepsilon)$-DP involve simply adding the $\varepsilon$ parameter of each mechanism \citep{mironov2017renyi}.

\begin{proposition}[Parallel composition] \label{prop:parallel_composition}
If $f:\mathcal{D} \to \mathcal{R}_1$, $g : \mathcal{D} \to \mathcal{R}_2$ are independent randomized maps that satisfy $(k_1,\alpha,\lambda/3,\varepsilon_1)$-RKRDP and $(k_2,\alpha,\lambda/3,\varepsilon_2)$-RKRDP, respectively,
then the mechanism $(f,g) : \mathcal{D} \to \mathcal{R}_1 \times \mathcal{R}_2$ defined as $D \mapsto (f(D), g(D))$
satisfies $(k,\alpha,\lambda,\varepsilon_1+\varepsilon_2)$-RKRDP, where $k((x,y),(x',y')) = k_1(x,x') k_2(y,y')$.
\end{proposition}

That is, for RKRDP the $\varepsilon$ guarantee of the composite mechanism is also the sum of the guarantees for each mechanism, but the regularization parameter degrades from $\lambda/3$ to $\lambda$ (recall that the RKRD decreases with $\lambda$). We also provide a result in which one of the mechanisms depends on the other.  

\begin{proposition}[Adaptive sequential composition] \label{prop:sequential_composition}
Let $\alpha \in (1,+\infty)$. Assume that $f:\mathcal{D} \to \mathcal{R}_1$ is a randomized map that satisfies $(\alpha,\varepsilon_1)$-RDP,
and that $g : \mathcal{R}_1 \times \mathcal{D} \to \mathcal{R}_2$ is a randomized map such that for any $x \in \mathcal{R}_1$, $g(x,\cdot)$ satisfies $(k,\alpha,\lambda,\varepsilon_2)$-RKRDP. Then the mechanism $h : \mathcal{D} \to \mathcal{R}_1 \times \mathcal{R}_2$ defined as $D \mapsto (f(D),g(D,f(D)))$
satisfies $(k,\alpha,\lambda,\varepsilon_1 + \varepsilon_2)$-RKRDP.
\end{proposition}

For $(\alpha,\varepsilon)$-RDP, the adaptive sequential composition guarantees are additive in $\varepsilon$, as in the parallel setting. \autoref{prop:sequential_composition} also shows additive $\varepsilon$ guarantees for RKRDP, but it assumes that the first map $f$ of the composition fulfills RDP instead of RKRDP. 

\section{Estimating the regularized kernel Rényi divergence from samples} \label{sec:estimating}

In practice, we usually do not have access to the densities of distributions; we are only given samples. In this setting, it is well-known that estimating the KL or Rényi divergence requires a number of samples $n$ exponential in the dimension  (\cite{krishnamurty2014nonparametric}; \cite{zhao2020minimax}, Theorems 4 and 5). In particular, this precludes the possibility of auditing RDP with practical sample sizes.

In this section, we see that it is possible to estimate the regularized kernel Rényi divergence using a moderate number of samples; we design an estimator of the RKRD between two distributions $p$ and $q$ given i.i.d samples ${(x_i)}_{i=1}^{n}$ and ${(y_i)}_{i=1}^{n}$ from $p$ and $q$ respectively, and we show a high-probability bound on the error of the estimator which falls as $O(1/\sqrt{n})$. 

We let $p_n = \frac{1}{n} \sum_{i=1}^{n} \delta_{x_i}$, $q_n = \frac{1}{n} \sum_{i=1}^{n} \delta_{y_i}$ be the empirical versions of $p$ and $q$. We define the approximate covariance operators $\hat{\Sigma}_p, \hat{\Sigma}_q$ as the covariance operators of $p_n$, $q_n$, i.e. $\hat{\Sigma}_p := \Sigma_{p_n}$, $\hat{\Sigma}_q := \Sigma_{q_n}$, and consider the estimator $D_{\alpha,\lambda}(\hat{\Sigma}_p||\hat{\Sigma}_q)$. The following theorem, proven in \autoref{sec:proofs_estimating}, bounds its error.

\begin{theorem} \label{thm:estimating}
Let $\alpha \geq 2$. For any $0 < x_0 < 1$, let
\begin{align} \label{eq:t_def}
    \ell = \log \left( \frac{14 \mathrm{Tr}(\Sigma_p - \Sigma_p^2)}{\|\Sigma_p - \Sigma_p^2\| x_0} \right), \quad t = \frac{\frac{\ell}{3} + \sqrt{(\frac{\ell}{3})^2 + 2 n \ell \|\Sigma_p - \Sigma_p^2\|}}{n}.
\end{align}
where $\|\cdot\|$ is the spectral norm.
With probability at least $1-x_0$, we have that 
\begin{align} \label{eq:concentration_ineq}
    |D_{\alpha,\lambda}(\hat{\Sigma}_p||\hat{\Sigma}_q) - D_{\alpha,\lambda}(\Sigma_p||\Sigma_q)| \leq B_n(x_0,\alpha,\lambda) := \frac{(\|\Sigma_q\| + (1+\frac{1}{\alpha}) \lambda)^{\alpha-1} (2 \alpha \lambda^{1-\alpha} + 4 (\alpha - 1))}{(\alpha-1)\mathrm{tr}[\Sigma_p^{\alpha}]} t.
\end{align}
\end{theorem}
Since $t = O(1/\sqrt{n})$, we obtain that $|D_{\alpha,\lambda}(\hat{\Sigma}_p||\hat{\Sigma}_q) - D_{\alpha,\lambda}(\Sigma_p||\Sigma_q)| = O(1/\sqrt{n})$ as promised. Also, remark that when $\lambda$ is small, $\lambda^{1-\alpha}$ is large, which makes the upper bound worse; In particular, the bound becomes vacuous for $\lambda = 0$. Hence, \autoref{thm:estimating} gives statistical guarantees for the estimation of the RKRD, not for the estimation of the KRD. In fact, the main reason to introduce the RKRD is its superior statistical performance. The proof of \autoref{thm:estimating} is based on a version of the Bernstein inequality for Hermitian operators (\autoref{thm:bernstein_operator}). 

The following proposition provides a tractable way to compute $D_{\alpha,\lambda}(\hat{\Sigma}_p||\hat{\Sigma}_q)$, or equivalently $\exp((\alpha-1)D_{\alpha,\lambda}(\hat{\Sigma}_p||\hat{\Sigma}_q))$, in terms of the samples ${(x_i)}_{i=1}^{n}$ and ${(y_i)}_{i=1}^{n}$.
\begin{proposition} \label{lem:empirical_RKRD}
Let $(x_i)_{i=1}^{n}$ be the samples from $p$ and $(y_i)_{i=1}^{n}$ be the samples from $q$, and let $K_{x,x}, K_{x,y}, K_{y,x}, K_{y,y} \in \R^{n \times n}$ be the matrices defined component-wise as $[K_{x,x}]_{i,j} = K(x_i,x_j)$, $[K_{x,y}]_{i,j} = K(x_i,y_j)$, $[K_{y,x}]_{i,j} = K(y_i,x_j)$, $[K_{y,y}]_{i,j} = K(y_i,y_j)$. We define the matrices $K_p, K_q \in \R^{2n \times 2n}$ block-wise as
\begin{align}
    K_p = 
    \begin{bmatrix}
    K_{x,x} & K_{x,y} \\
    0 & 0
    \end{bmatrix} \qquad 
    K_q = 
    \begin{bmatrix}
    0 & 0 \\
    K_{y,x} & K_{y,y}
    \end{bmatrix}.
\end{align}
We obtain that
\begin{align}
\begin{split}
    &\mathrm{tr}\left[\left( (\hat{\Sigma}_q + \lambda \mathrm{Id})^{\frac{1-\alpha}{2\alpha}} \hat{\Sigma}_p (\hat{\Sigma}_q + \lambda \mathrm{Id})^{\frac{1-\alpha}{2\alpha}} \right)^\alpha \right] \\ &= \mathrm{tr}\bigg[\bigg( \left(\frac{1}{n}K_q + \lambda \mathrm{Id} \right)^{\frac{1-\alpha}{2\alpha}} \frac{1}{n}K_p \left(\frac{1}{n}K_q + \lambda \mathrm{Id} \right)^{\frac{1-\alpha}{2\alpha}} \bigg)^\alpha \bigg].
\end{split}
\end{align}
\end{proposition}
That is, $D_{\alpha,\lambda}(\hat{\Sigma}_p||\hat{\Sigma}_q) = \frac{1}{\alpha-1} \log \big( \mathrm{tr} \big[ \big( \left(\frac{1}{n}K_q + \lambda \mathrm{Id} \right)^{\frac{1-\alpha}{2\alpha}} \frac{1}{n}K_p \left(\frac{1}{n}K_q + \lambda \mathrm{Id} \right)^{\frac{1-\alpha}{2\alpha}} \big)^\alpha \big] \big)$.
Assuming that kernel evaluations take constant time, the time complexity is $O(n^3)$, because we need to compute the eigenvalue decomposition of $K_q$ to obtain the matrix power $\left(\frac{1}{n}K_q + \lambda \mathrm{Id} \right)^{\frac{1-\alpha}{2\alpha}}$, and then the product $\left(\frac{1}{n}K_q + \lambda \mathrm{Id} \right)^{\frac{1-\alpha}{2\alpha}} \frac{1}{n}K_p \left(\frac{1}{n}K_q + \lambda \mathrm{Id} \right)^{\frac{1-\alpha}{2\alpha}}$ and its eigenvalue decomposition.

\section{Auditing differential privacy with the regularized kernel Rényi divergence} \label{sec:auditing}
\textbf{A hypothesis test for RKRDP.} \autoref{sec:estimating} provides an estimator of the RKRD which can be used to audit RKRDP. Namely, suppose that we have access to samples ${(x_i)}_{i=1}^{n}$ and ${(y_i)}_{i=1}^{n}$ from the random variables $f(D)$, $f(D')$, where $D,D' \in \mathcal{D}$ are adjacent and $f : \mathcal{D} \to \mathcal{R}$ is a randomized map that satisfies $(k,\alpha,\lambda,\varepsilon)$-RKRDP. By definition, we have that $D_{\alpha,\lambda}(\Sigma_{\mathcal{L}(f(D))}||\Sigma_{\mathcal{L}(f(D'))}) \leq \varepsilon$. Applying \autoref{thm:estimating}, we obtain that with probability at least $1-x_0$,
\begin{align} \label{eq:test}
    D_{\alpha,\lambda}(\hat{\Sigma}_p||\hat{\Sigma}_q) \leq \varepsilon + B_n(x_0,\alpha,\lambda),
\end{align}
where $B_n(x_0,\alpha,\lambda) = O(1/\sqrt{n})$ is the bound defined in \eqref{eq:concentration_ineq}. Formally, this gives rise to a hypothesis test of level $x_0$ for $(k,\alpha,\lambda,\varepsilon)$-RKRDP: the null hypothesis is that RKRDP holds, the test statistic is $D_{\alpha,\lambda}(\hat{\Sigma}_p||\hat{\Sigma}_q)$ and the threshold is $\varepsilon + B_n(x_0,\alpha,\lambda)$; if the test statistic is larger than the threshold we reject the null hypothesis. The test has level $x_0$ because under the null hypothesis, the rejection probability (the type I error) is below $x_0$. An analysis of the power (i.e. the type II error) of this hypothesis test is out of the scope of this paper. 

\textbf{Testing for $(\alpha,\varepsilon)$-RDP and $\varepsilon$-DP in high dimensions.} As mentioned at the end of \autoref{sec:framework}, RKRDP is a relaxation of $(\alpha,\varepsilon)$-Rényi differential privacy, and \textit{a fortiori}, of $\varepsilon$-differential privacy. Hence, the test \eqref{eq:test} is also a hypothesis test of level $x_0$ for $(\alpha,\varepsilon)$-RDP and $\varepsilon$-DP. We reemphasize that in high dimensions, $(\alpha,\varepsilon)$-RDP cannot be audited directly using a test statistic which estimates the Rényi divergence from samples, because all such estimators are cursed by dimension (see \autoref{sec:estimating}).

\textbf{Testing for $(\varepsilon,\delta)$-DP in high dimensions.} We can further regard the test \eqref{eq:test} as a hypothesis test of level $x_0$ for $(\varepsilon,\delta)$-differential privacy by choosing the regularization parameter $\lambda$ appropriately. Namely, we use the following proposition:
\begin{proposition} \label{thm:epsilon_delta_implies}
For any $\varepsilon, \delta > 0$, $\alpha \in [1/2)\cup(1,+\infty)$ and any kernel $k$, $(\varepsilon,\delta)$-differential privacy implies $(k,\alpha, \delta e^{-\varepsilon},\varepsilon)$-regularized kernel Rényi differential privacy.
\end{proposition}
\autoref{thm:epsilon_delta_implies}, which is proven in \autoref{sec:proofs_auditing}, shows that $(k,\alpha,\lambda,\varepsilon)$-RKRDP is a relaxation of $(\varepsilon,\delta)$-DP when $\lambda = \delta e^{-\varepsilon}$. Thus, the test \eqref{eq:test} given by $D_{\alpha,\delta e^{-\varepsilon}}(\hat{\Sigma}_p||\hat{\Sigma}_q)$ has level $x_0$ for $(\varepsilon,\delta)$-DP. 
There does not seem to be any alternative approaches to audit $(\varepsilon,\delta)$-DP in high dimensions. Note that while \cite{mironov2017renyi} (Proposition 3) showed that $(\alpha,\varepsilon)$-RDP implies $(\varepsilon + \log(1/\delta)/(\alpha-1), \delta)$-DP, this implication cannot be leveraged into a hypothesis test for $(\varepsilon,\delta)$-DP even when a test for $(\alpha,\varepsilon)$-RDP is available (such as the one described in this section).
When testing for any of the mentioned DP notions, the choice of pairs of adjacent databases $D, D' \in \mathcal{D}$ can be done as described in \cite{Ding2018} and \cite{JagielskiUO20}.

\section{Experiments}\label{sec:exp}
We perform experiments in which we apply the Gaussian mechanism, and we compute the RKRD estimator $D_{\alpha,\lambda}(\hat{\Sigma}_p||\hat{\Sigma}_q)$ to audit DP notions. 

For a deterministic map $f : \mathcal{D} \to \R^d$, we define the $L^2$ sensitivity $\Delta = \sup_{D, D' \in \mathcal{D}, D \sim D'} \|f(D)-f(D')\|_2$, where the supremum is over adjacent elements. Given a deterministic map $f : \mathcal{D} \to \R^d$ with $L^2$ sensitivity $\Delta$ and parameters $\varepsilon, \delta > 0$, the Gaussian mechanism \citep{dwork2014thealgorithmic} consists in adding multivariate Gaussian noise $N(0,\sigma^2 \mathrm{Id})$ to the output $f(D)$, where $\sigma = \Delta \sqrt{2 \log(1.25/\delta)}/\varepsilon$ is chosen such that the mechanism satisfies $(\varepsilon,\delta)$-DP. 
Later on, \cite{balle2018improving} showed an exact characterization of the values $\sigma$ such that $(\varepsilon,\delta)$-DP holds: letting $\Phi$ be the Gaussian CDF, $(\varepsilon,\delta)$-DP is equivalent to
$\Phi(\frac{\Delta}{2 \sigma} - \frac{\varepsilon \sigma}{\Delta}) - e^{\varepsilon} \Phi (-\frac{\Delta}{2 \sigma} - \frac{\varepsilon \sigma}{\Delta} ) \leq \delta$.
Unlike for the Laplace mechanism, thanks to the rotational symmetry of the multivariate Gaussian distribution, the only relevant parameter when comparing the outputs of the Gaussian mechanism on $D, D'$ is the ratio between $\sigma$ and $\|f(D)-f(D')\|_2$. 

In our experiments, we take $d = 30$, $\Delta = \|f(D)-f(D')\|_2 = 0,10$. For $\Delta = 10$, we consider three pairs of $\varepsilon$, $\delta$, and for each we set $\sigma$ as small as possible such that $(\varepsilon,\delta)$-DP holds according to the characterization by \cite{balle2018improving} 
(for $\Delta = 0$ we just set $\sigma = 0.01$). We compute the RKRD for different values of $\lambda$ and $\alpha$, and several sample sizes $n$; the plots are shown in \autoref{fig:rkrd_audit}. At first sight, we see that the RKRD increases with $\alpha$ and decreases with $\lambda$ as shown in \autoref{prop:comparison_renyi}. Also, in accordance with \autoref{thm:estimating}, the RKRD estimators with different numbers of samples are close (and presumably close to the real value), more so for higher $\lambda$. We provide the experimental details in \autoref{sec:exp_details}. The code can be found at \url{https://github.com/CDEnrich/kernel_renyi_dp}.
\begin{figure}[ht!]
    \centering
    \includegraphics[width=0.48\textwidth]{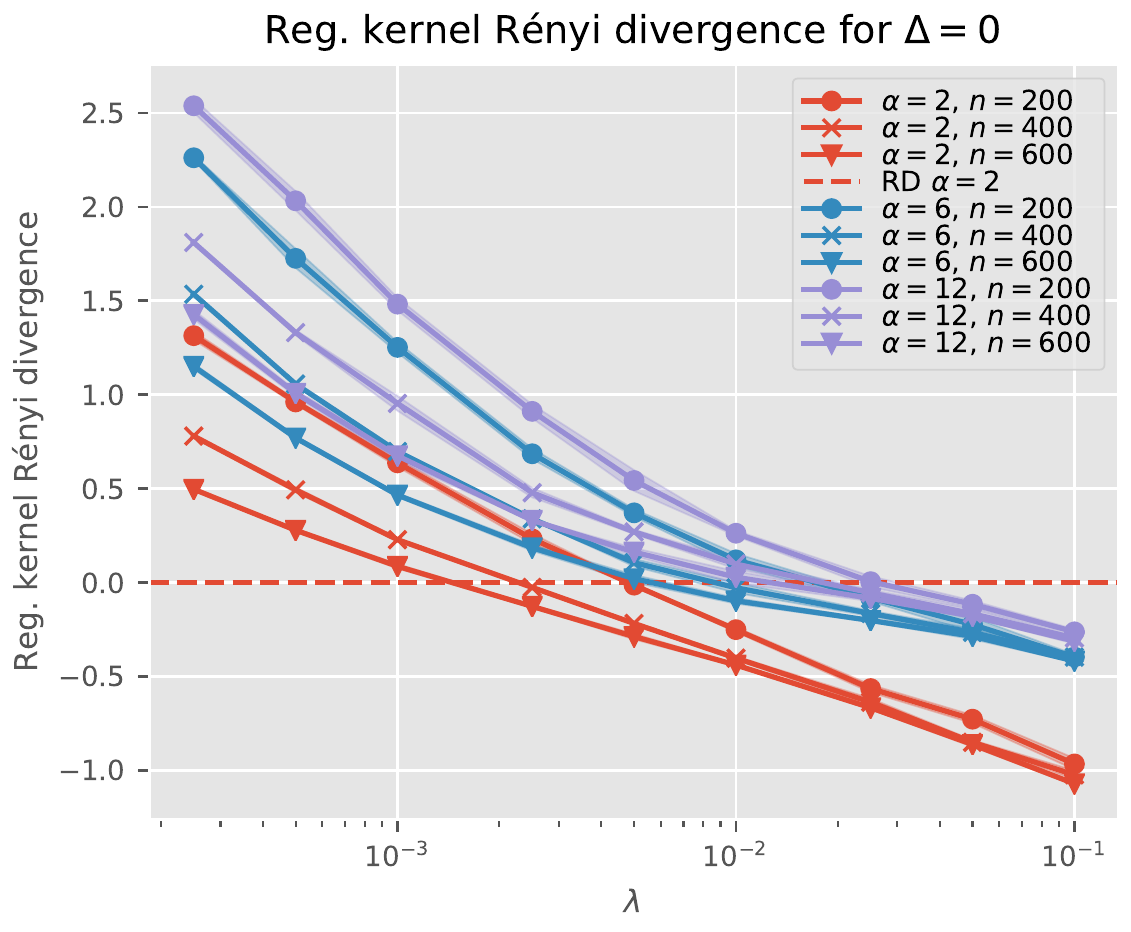}
    \includegraphics[width=0.49\textwidth]{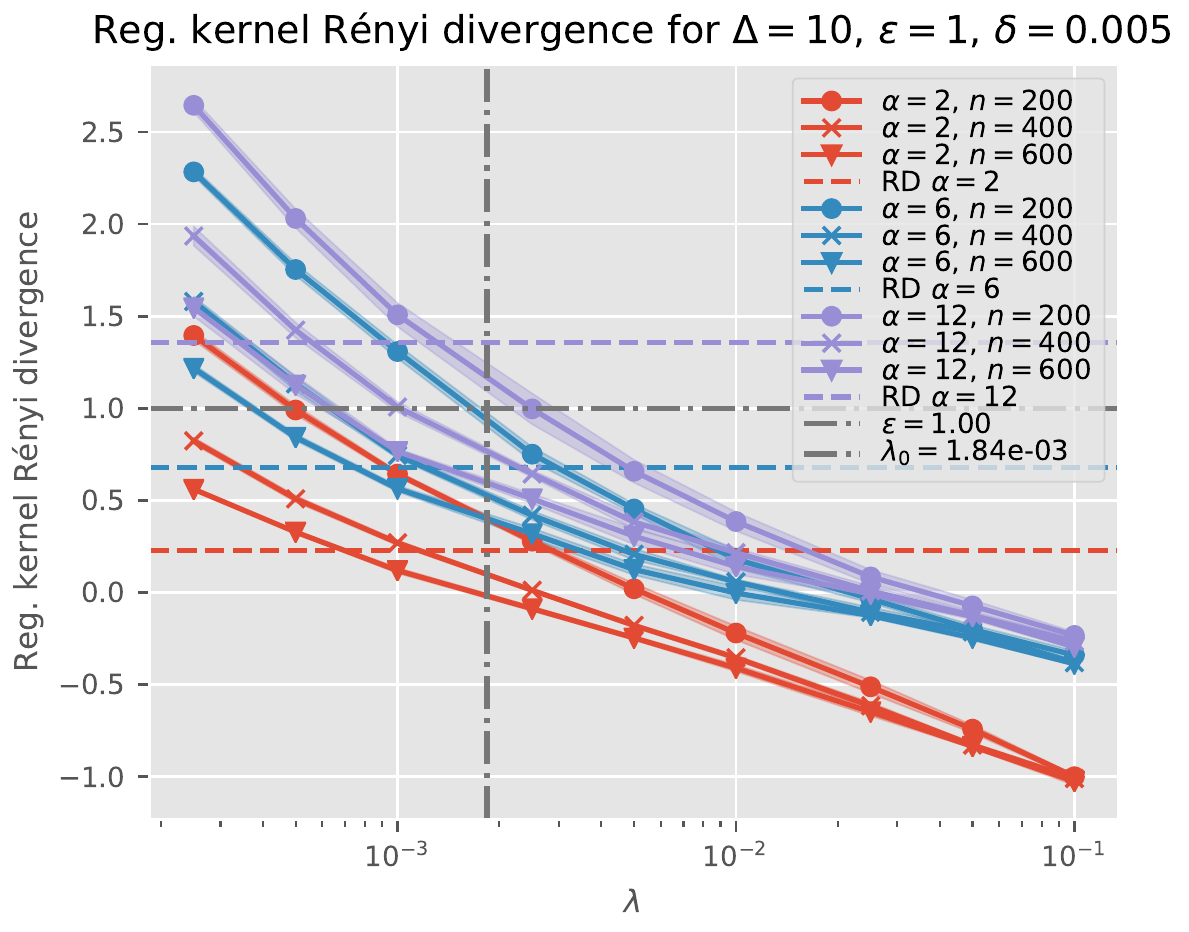} \\
    \includegraphics[width=0.49\textwidth]{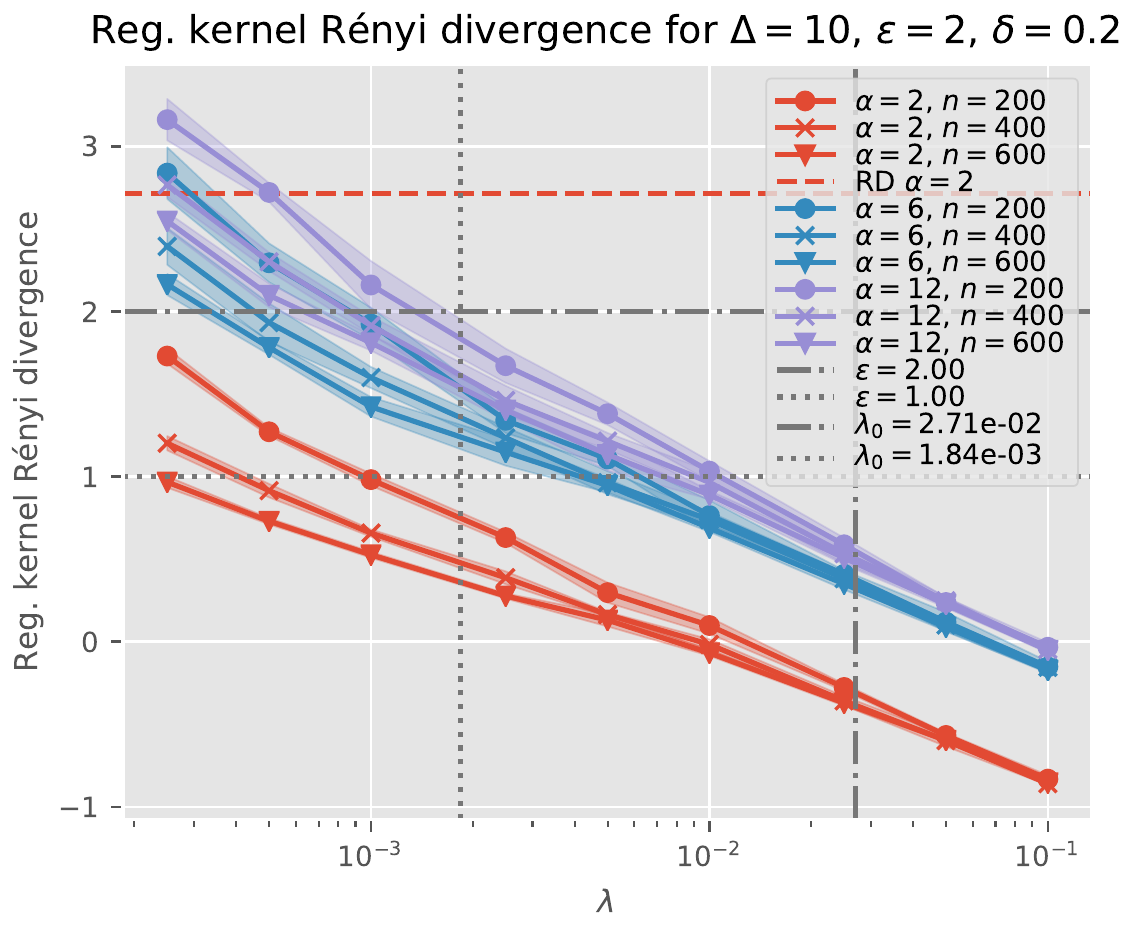}
    \includegraphics[width=0.49\textwidth]{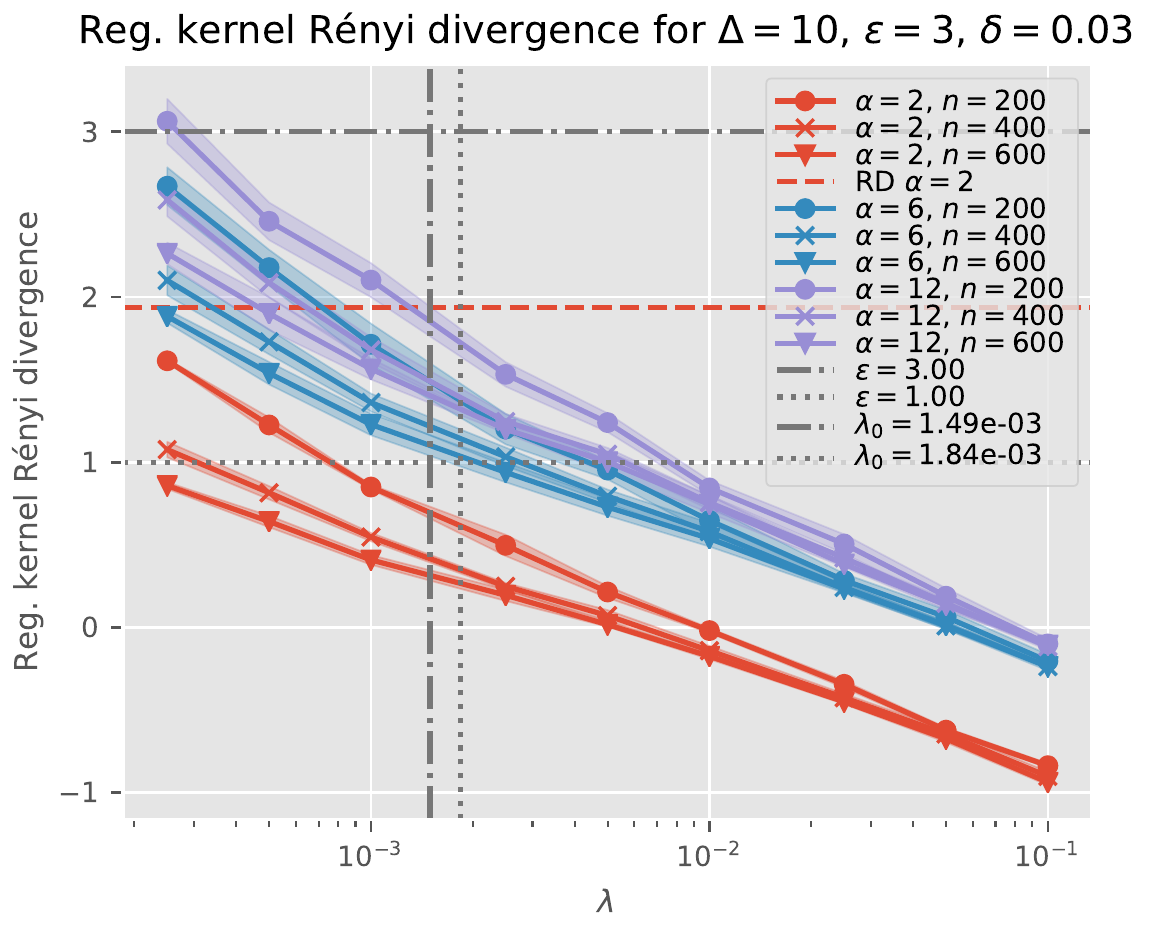}
    \caption{RKRD values for the Gaussian mechanism for $d=30$. Values of noise variance $\sigma^2$: 0.0001 (top left), 21.0444 (top right), 6.0669 (bottom left), 7.1850 (bottom right). RD values not shown: for $\alpha=6$; 8.1506 (bottom left), 5.8112 (bottom right), for $\alpha = 12$; 16.3013 (bottom left), 11.6224 (bottom right). Means and standard deviations computed from 5 runs. See explanation in the text.}
    \label{fig:rkrd_audit}
\end{figure}

\textbf{Testing for $(\varepsilon,\delta)$-DP.} The horizontal and vertical dash-dotted lines in the top right plot of \autoref{fig:rkrd_audit} show $\varepsilon = 1$ and $\lambda_0 = \delta e^{-\varepsilon} = \num{1.84e-3}$, where $\delta = 0.005$. Since $\sigma$ is chosen such that $(\varepsilon,\delta)$-DP holds, \autoref{thm:epsilon_delta_implies} implies that $(k,\alpha, \lambda_0,\varepsilon)$-RKRDP holds, i.e. the RKRD with $\lambda_0$ and any $\alpha$ is smaller than $\varepsilon$. This can be observed graphically; for $n = 600$ and $\alpha = 2, 6, 12$, the RKRD curves lie below the intersection of dash-dotted lines. In the bottom plots, the dash-dotted lines represent $\varepsilon$ and $\lambda_0 = \delta e^{-\varepsilon}$ for their corresponding $\varepsilon, \delta$ pairs (2, 0.2 and 3, 0.03, respectively). Additionally, we use dotted lines to show $\varepsilon$ and $\lambda_0$ for $\varepsilon = 1, \delta = 0.005$. Since in both cases the intersection of dotted lines lies well below the RKRD curves for $\alpha = 12$ (so much so that the statistical error is not an issue), we can reject the hypothesis that these two mechanisms satisfy $(\varepsilon,\delta)$-DP with $\varepsilon = 1, \delta = 0.005$. 

\textbf{Testing for $(\alpha,\varepsilon)$-RDP.} The dashed lines in the plots show the values of $\alpha$-RD for different $\alpha$, which admit the analytic expression $\frac{\alpha \Delta^2}{2 \sigma^2}$ for isotropic Gaussians \citep{gil2013renyi}. When the statistical error is small ($n$ high enough), we see that the RKRD curves lie below the corresponding RD line, as expected. This provides a way to test for $(\alpha,\varepsilon)$-RDP.

\section{Conclusion}
We introduced in this paper a new notion of differential privacy the regularized  kernel Rényi differential privacy based on the kernel Rényi divergence and showed that its estimation enjoys parametric rates and that it leads to a principled auditing of violations of differential privacy. Future work includes learning the kernel by taking the supremum of the regularized kernel Rényi divergence on a class of kernels or feature maps inducing a kernel, as well as exploring the link of the learned kernel with  adversaries with limited capacity  \citep{chaudhuri2019capacity}.




\bibliography{biblio}


\appendix

\tableofcontents

\addtocontents{toc}{\protect\setcounter{tocdepth}{2}}

\section{Proofs of \autoref{sec:framework}} \label{sec:proofs_framework}

\textit{\textbf{Proof of \autoref{prop:prop_regularized}.}}
Properties 1, 2, 3 and 4 are immediate given the corresponding properties for the quantum Rényi divergence (\cite{mullerlennert2013onquantum}, Theorem 2).
To show 5, remark that $B \otimes D + \lambda \mathrm{Id}_{\mathcal{H}_1 \otimes \mathcal{H}_2} \leq B \otimes D + \sqrt{\lambda} B \otimes \mathrm{Id}_{\mathcal{H}_1} + \sqrt{\lambda} \mathrm{Id}_{\mathcal{H}_2} \otimes D + \lambda \mathrm{Id}_{\mathcal{H}_1 \otimes \mathcal{H}_2} = (B + \sqrt{\lambda} \mathrm{Id}_{\mathcal{H}_1}) \otimes (D + \sqrt{\lambda} \mathrm{Id}_{\mathcal{H}_2})$. Hence, applying \autoref{prop:4_lennert}, we obtain that
\begin{align}
    D_{\alpha}(A \otimes C||B \otimes D + \lambda \mathrm{Id}) &\geq D_{\alpha}(A \otimes C||(B + \sqrt{\lambda} \mathrm{Id}_{\mathcal{H}_1}) \otimes (D + \sqrt{\lambda} \mathrm{Id}_{\mathcal{H}_2})) \\ &= D_{\alpha}(A||B + \sqrt{\lambda} \mathrm{Id}_{\mathcal{H}_1}) + D_{\alpha}(C||D + \sqrt{\lambda} \mathrm{Id}_{\mathcal{H}_2}).
\end{align}
The last equality holds because of the additivity property of the quantum Rényi divergence (\cite{mullerlennert2013onquantum}, Theorem 2).
To show the other inequality, we write 
\begin{align}
\left(B + \frac{\lambda}{3} \mathrm{Id} \right) \otimes \left(D + \frac{\lambda}{3} \mathrm{Id} \right) = B \otimes D + \frac{\lambda}{3} \mathrm{Id} \otimes D + \frac{\lambda}{3} B \otimes \mathrm{Id} + \frac{\lambda^2}{9} \mathrm{Id} \otimes \mathrm{Id} \leq B \otimes D + \lambda \mathrm{Id} \otimes \mathrm{Id},
\end{align}
where we used that $B, D\leq \mathrm{Id}$, 
and that $\frac{\lambda^2}{9} \leq \frac{\lambda}{3}$ because $\lambda \leq 3$.
To show 6, we use that $B \oplus D + \lambda \mathrm{Id}_{\mathcal{H}_1 \oplus \mathcal{H}_2} = (B + \lambda \mathrm{Id}_{\mathcal{H}_1}) \oplus (D + \lambda \mathrm{Id}_{\mathcal{H}_2})$, which implies that
\begin{align}
    &D_{\alpha}(A \oplus C||B \oplus D + \lambda \mathrm{Id}_{\mathcal{H}_1 \oplus \mathcal{H}_2}) = D_{\alpha}(A \oplus C||(B + \lambda \mathrm{Id}_{\mathcal{H}_1}) \oplus (D + \lambda \mathrm{Id}_{\mathcal{H}_2})) \\ &= g^{-1} \bigg( \frac{\mathrm{tr}[A]}{\mathrm{tr}[A+C]} g(D_{\alpha}(A||B + \lambda \mathrm{Id}_{\mathcal{H}_1})) + \frac{\mathrm{tr}[C]}{\mathrm{tr}[A+C]} g(D_{\alpha}(C||D + \lambda \mathrm{Id}_{\mathcal{H}_2})) \bigg), 
\end{align}
where second inequality holds by the mean property of the quantum Rényi divergence.
\qed

\textit{\textbf{Proof of \autoref{prop:comparison_renyi}.}}
To prove $D_{\alpha,\lambda_1}(\Sigma_p||\Sigma_q) \geq D_{\alpha,\lambda_2}(\Sigma_p||\Sigma_q)$, we use that $\Sigma_q + \lambda_1 \mathrm{Id} \geq \Sigma_q + \lambda_2 \mathrm{Id}$, which means that $D_{\alpha,\lambda_1}(\Sigma_p||\Sigma_q) = D_{\alpha}(\Sigma_p||\Sigma_q + \lambda_1 \mathrm{Id}) \leq D_{\alpha}(\Sigma_p||\Sigma_q + \lambda_2 \mathrm{Id}) = D_{\alpha,\lambda_2}(\Sigma_p||\Sigma_q)$, where we used \autoref{prop:4_lennert} and that $\alpha \in [1/2,+\infty)$. 

To prove $D_{\alpha}(p||q) \geq D_{\alpha,\lambda}(\Sigma_p||\Sigma_q)$, we use that for $\alpha \in [1/2,1)\cup(1,+\infty)$ by \autoref{lem:frank_lieb_lemma}(ii), $(A,B) \mapsto \exp((\alpha-1)D_{\alpha}(A||B))$ is convex when $A$ has a fixed trace and thus,
\begin{align}
    &\exp((\alpha-1)D_{\alpha}(\Sigma_p || \Sigma_q + \lambda \mathrm{Id})) \\
    &= \exp\bigg((\alpha-1)D_{\alpha}\bigg(\int_{\mathcal{R}} \phi(x) \phi^{*}(x) \frac{dp}{dq}(x) \, dq(x) \bigg|\bigg| \int_{\mathcal{R}} \phi(x) \phi^{*}(x) \, dq(x) + \lambda \mathrm{Id} \bigg) \bigg) \\ &\leq \int_{\mathcal{R}} \exp\bigg((\alpha-1)D_{\alpha}\bigg(\phi(x) \phi^{*}(x) \frac{dp}{dq}(x) \bigg|\bigg| \phi(x) \phi^{*}(x) + \lambda \mathrm{Id} \bigg) \bigg) \, dp(x) 
    \\ &\leq \int_{\mathcal{R}} \exp\bigg((\alpha-1)D_{\alpha}\bigg(\phi(x) \phi^{*}(x) \frac{dp}{dq}(x) \bigg|\bigg| \phi(x) \phi^{*}(x) \bigg) \bigg) \, dp(x)
    \\ &= \int_{\mathcal{R}} \exp\bigg((\alpha-1) \log\left(\frac{dp}{dq}(x) \right) \bigg) \, dp(x) = \int_{\mathcal{R}} \left(\frac{dp}{dq}(x) \right)^{\alpha-1} \, dp(x) \\ &= \int_{\mathcal{R}} \left(\frac{dp}{dq}(x) \right)^{\alpha} \, dq(x) = \exp((\alpha-1)D_{\alpha}(p||q)).
\end{align}
The second inequality holds by \autoref{prop:4_lennert}.
We have also used that by definition, $D_{\alpha}(\phi(x) \phi^{*}(x) \frac{dp}{dq}(x) || \phi(x) \phi^{*}(x) )$ is equal to
\begin{align}
 &\frac{\log \left( \left( \frac{dp}{dq}(x) \right)^{-1} \mathrm{tr}[\phi(x) \phi(x)^{*}]^{-1} \bigg(\frac{dp}{dq}(x) \bigg)^{\alpha} \mathrm{tr}\bigg[ \left( (\phi(x) \phi^{*}(x))^{\frac{1-\alpha}{2\alpha}} \phi(x) \phi^{*}(x) (\phi(x) \phi^{*}(x))^{\frac{1-\alpha}{2\alpha}} \right)^\alpha \bigg] \right)}{\alpha-1} \\ &= \frac{1}{\alpha-1} \log\left(\frac{dp}{dq}(x) \right)^{\alpha-1} + D_{\alpha}(\phi(x) \phi^{*}(x) || \phi(x) \phi^{*}(x)) = \log\left(\frac{dp}{dq}(x) \right).
\end{align}
To prove that $D_{\alpha_2,\lambda}(\Sigma_p||\Sigma_q) \leq D_{\alpha_1,\lambda}(\Sigma_p||\Sigma_q)$ for $\alpha_1 \geq \alpha_2$, we use \autoref{lem:monotonicity_alpha}, which implies that when $\alpha_1 \geq \alpha_2$, $D_{\alpha_1}(\Sigma_{f(D)}||\Sigma_{f(D')} + \lambda \mathrm{Id}) \geq D_{\alpha_2}(\Sigma_{f(D)}||\Sigma_{f(D')} + \lambda \mathrm{Id})$. 
\qed

\begin{lemma}[\cite{mullerlennert2013onquantum}, Proposition 4] \label{prop:4_lennert}
Let $A \neq 0$ and $B \geq B' \geq 0$. Then, for $\alpha \in [1/2,1) \cup (1,+\infty)$, we have $D_{\alpha}(A || B) \leq D_{\alpha}(A || B')$.
\end{lemma}

\begin{lemma}[Monotonicity in $\alpha$, \cite{mullerlennert2013onquantum},Theorem 7] \label{lem:monotonicity_alpha}
Let $A, B \geq 0$ with $A \neq 0$. Then, $\alpha \mapsto D_{\alpha}(A || B)$ is monotonically increasing for $\alpha \in (0,1)\cup(1,+\infty)$.
\end{lemma}

\begin{lemma} [\cite{frank2013monotonicity}, Prop. 3] \label{lem:frank_lieb_lemma}
(i) The following map on pairs of positive operators 
\begin{align}
    (A,B) \mapsto \mathrm{tr}\left[ \left( B^{\frac{1-\alpha}{2\alpha}} A B^{\frac{1-\alpha}{2\alpha}} \right)^{\alpha} \right]
\end{align}
is jointly concave for $1/2 \leq \alpha < 1$ and jointly convex for $\alpha > 1$. 
(ii) This implies that on the set of pairs of positive operators $(A,B)$ such that $\mathrm{tr}(A)$ has a fixed value, $(A,B) \mapsto \exp((\alpha-1)D_{\alpha}(A||B))$ is jointly convex for any $\alpha \in [1/2,1)\cup(1,+\infty)$.
\end{lemma}

\textit{\textbf{Proof of \autoref{prop:krdp_implications}.}}
To prove 1, we use that if $+\infty \geq \lambda_1 \geq \lambda_2 \geq 0$, $D_{\alpha,\lambda_1}(\Sigma_p||\Sigma_q) \leq D_{\alpha,\lambda_2}(\Sigma_p||\Sigma_q)$ by \autoref{prop:comparison_renyi}.

To prove 2, we use that if $\alpha_1 \geq \alpha_2$, $D_{\alpha_1,\lambda}(\Sigma_{f(D)}||\Sigma_{f(D')}) \geq D_{\alpha_2,\lambda}(\Sigma_{f(D)}||\Sigma_{f(D')})$ by \autoref{prop:comparison_renyi}. 

To prove 3, we use that for $\alpha \in [1/2,1)\cup(1,+\infty)$, $D_{\alpha,\lambda}(\Sigma_p || \Sigma_q) \leq D_{\alpha}(p||q)$ by \autoref{prop:comparison_renyi}.

To prove 4, we focus on the case $J=2$, as the general case follows from finite induction. Suppose that $\phi_1 : \mathcal{R} \to \mathcal{H}_1$, $\phi_2 : \mathcal{R} \to \mathcal{H}_2$ are the feature maps for $k_1$ and $k_2$. Then, the kernel $k$ corresponds to the RKHS $\mathcal{H} = \mathcal{H}_1 \otimes \mathcal{H}_2$, and it has feature map $\phi = \phi_1 \otimes \phi_2 : \mathcal{R} \to \mathcal{H}$ defined as $\phi(x) = \phi_1(x) \otimes \phi_2(x)$. Thus, for any probability measure $p$, we have
\begin{align} \label{eq:sigma_H_sigma_H12}
\Sigma^{\mathcal{H}}_{p} &:= \int_{\mathcal{R}} (\phi_1(x) \otimes \phi_2(x)) (\phi_1(x) \otimes \phi_2(x))^{*} \, dp(x) = \int_{\mathcal{R}} \phi_1(x) \phi_1(x)^{*} \otimes \phi_2(x) \phi_2(x)^{*} \, dp(x) \\ &= \bigg(\int_{\mathcal{R}} \phi_1(x) \phi_1(x)^{*} \, dp(x) \bigg) \otimes \bigg(\int_{\mathcal{R}} \phi_2(x) \phi_2(x)^{*} \, dp(x) \bigg) := \Sigma^{\mathcal{H}_1}_{p} \otimes \Sigma^{\mathcal{H}_2}_{p}.
\end{align}
By the additivity property 5 of the regularized quantum Rényi divergence (\autoref{prop:prop_regularized}), we obtain that
\begin{align} \label{eq:additivity_QRD}
    D_{\alpha,\lambda}(\Sigma^{\mathcal{H}}_{p} || \Sigma^{\mathcal{H}}_{q}) = D_{\alpha,\lambda}(\Sigma^{\mathcal{H}_1}_{p} \otimes \Sigma^{\mathcal{H}_2}_{p} || \Sigma^{\mathcal{H}_1}_{q} \otimes \Sigma^{\mathcal{H}_2}_{q}) \geq D_{\alpha,\sqrt{\lambda}}(\Sigma^{\mathcal{H}_1}_{p} || \Sigma^{\mathcal{H}_1}_{q}) + D_{\alpha,\sqrt{\lambda}}(\Sigma,\lambda^{\mathcal{H}_2}_{p} || \Sigma^{\mathcal{H}_2}_{q}).
\end{align}
\qed

\section{Proofs of \autoref{sec:properties}} \label{sec:proofs_properties}

\textbf{\textit{Proof of \autoref{prop:convexity}.}} 
For arbitrary adjacent $D, D' \in \mathcal{D}$, let $p = \mathcal{L}(f(D))$, $p' = \mathcal{L}(g(D))$, $q = \mathcal{L}(f(D'))$, $q' = \mathcal{L}(g(D'))$.
We rely on \autoref{lem:frank_lieb_lemma}(ii), which states that
\begin{align}
    \beta \mapsto \exp((\alpha - 1) D_{\alpha}(\beta \Sigma_{p'} + (1-\beta) \Sigma_{p}||\beta \Sigma_{q'} + (1-\beta) \Sigma_{q} + \lambda \mathrm{Id}))
\end{align}
is a convex function for $\alpha \in [1/2,1)\cup(1,+\infty)$. Hence, for all $\beta \in [0,1]$,
\begin{align}
    &\exp((\alpha - 1) D_{\alpha}(\beta \Sigma_{p'} + (1-\beta) \Sigma_{p}||\beta \Sigma_{q'} + (1-\beta) \Sigma_{q} + \lambda \mathrm{Id})) \\ &\leq \beta \exp((\alpha - 1) D_{\alpha}(\Sigma_{p'}||\Sigma_{q'} + \lambda \mathrm{Id})) + (1-\beta)\exp((\alpha - 1) D_{\alpha}(\Sigma_{p}||\Sigma_{q} + \lambda \mathrm{Id})) \\ &\leq \exp((\alpha-1) \varepsilon).
\end{align}
Since 
\begin{align}
\Sigma_{\mathcal{L}(h(D))} &= \int_{\mathcal{R}} \phi(x) \phi(x)^{*} \, d\mathcal{L}(h(D))(x) \\ &= \beta \int_{\mathcal{R}} \phi(x) \phi(x)^{*} \, d\mathcal{L}(f(D))(x) + (1-\beta) \int_{\mathcal{R}} \phi(x) \phi(x)^{*} \, d\mathcal{L}(g(D))(x) \\ &= \beta \Sigma_{\mathcal{L}(f(D))} + (1-\beta) \Sigma_{\mathcal{L}(g(D))},
\end{align}
we obtain that $\exp((\alpha - 1) D_{\alpha,\lambda}(\Sigma_{\mathcal{L}(h(D))}||\Sigma_{\mathcal{L}(h(D'))} + \lambda \mathrm{Id})) \leq \exp((\alpha-1) \varepsilon)$, which implies that $D_{\alpha,\lambda}(\Sigma_{\mathcal{L}(h(D))}||\Sigma_{\mathcal{L}(h(D'))} + \lambda \mathrm{Id}) \leq \varepsilon$ when $\alpha > 1$.
\qed

\subsection{Results and proofs of the postprocessing properties} \label{subsec:proofs_postprocessing}

First, we introduce some concepts and necessary results. 
\paragraph{Compact operators and trace-class operators.} Let $\mathcal{H}, \mathcal{H}'$ be Hilbert spaces. 

\begin{definition}[Compact operators, trace-class operators on $\mathcal{H}$]
Let $\mathcal{K}(\mathcal{H})$ be the space of compact linear operators on a Hilbert space $\mathcal{H}$, which contains the operators $A : \mathcal{H} \to \mathcal{H}$ of the form
\begin{align} \label{eq:compact_operator}
    A = \sum_{n=1}^{N} \rho_n \langle f_n, \cdot \rangle g_n, \quad \text{where } 0 \leq N \leq +\infty, \, \rho_n \geq 0, \, (f_n), (g_n) \text{ orthonormal sets of } \mathcal{H}. 
\end{align}
The operator norm on $\mathcal{K}(\mathcal{H})$ is defined as $\|A\| = \sup_{h \in \mathcal{H}, \|h\| \leq 1} \|Ah\|$. When $A$ admits the expression \eqref{eq:compact_operator}, $\|A\| = \sup_{n} \rho_n$.

The dual space of $\mathcal{K}(\mathcal{H})$, denoted as $\mathcal{K}^{*}(\mathcal{H})$, is known as the space of trace-class (or nuclear) linear operators, and it contains the linear operators $A : \mathcal{H} \to \mathcal{H}$ of the form
\begin{align}
    A = \sum_{i=1}^{N} \rho_n \langle f_n, \cdot \rangle g_n, \quad \text{where } 0 \leq N \leq +\infty, \, \rho_n \geq 0, \, (f_n), (g_n) \text{ orthonormal sets of } \mathcal{H},
\end{align}
such that $\sum_{i=1}^{n} \rho_n < +\infty$. Given an arbitrary orthonormal basis $(e_i)_{i=1}^{+\infty}$ of $\mathcal{H}$, the quantity $\mathrm{tr}(A) := \sum_{i=1}^{+\infty} \langle e_i, A e_i \rangle$ is well defined and is known as the trace of $A$. The dual norm of $A \in \mathcal{K}^{*}(\mathcal{H})$ is defined as
\begin{align}
    \|A\| = \sup_{B \in \mathcal{K}(\mathcal{H}), \|B\| \leq 1} \mathrm{tr}(A B).
\end{align}
\end{definition}

$\mathcal{K}^{*}(\mathcal{H})$ is the dual space of $\mathcal{K}(\mathcal{H})$ in the sense that any linear form from $\mathcal{K}(\mathcal{H})$ to $\R$ can be written as $A \mapsto \mathrm{tr}(B A)$, with $B \in \mathcal{K}^{*}(\mathcal{H})$. 

We can define the dual operator of an operator from $\mathcal{K}(\mathcal{H})$ to $\mathcal{K}(\mathcal{H}')$ as follows:
\begin{definition}[Dual operator of an operator from $\mathcal{K}(\mathcal{H})$ to $\mathcal{K}(\mathcal{H}')$]
Given an operator $T: \mathcal{K}(\mathcal{H}) \to \mathcal{K}(\mathcal{H}')$, the dual operator $T^{*} : \mathcal{K}^{*}(\mathcal{H}') \to \mathcal{K}^{*}(\mathcal{H})$ is defined as
\begin{align}
    \forall B \in \mathcal{K}^{*}(\mathcal{H}'), \forall A \in \mathcal{K}(\mathcal{H}), \quad \langle T^{*}(B), A \rangle
    = \langle B, T(A) \rangle
\end{align}
Here, for $A \in \mathcal{K}^{*}(\mathcal{H})$, $B \in \mathcal{K}(\mathcal{H})$ we used the notation $\langle A, B \rangle
= \mathrm{tr}(A^{*} B)$ (and similarly for $\mathcal{H}'$). 
\end{definition}
\paragraph{Completely positive trace-preserving maps and the data processing inequality.} The next definition introduces a class of maps from $\mathcal{K}^{*}(\mathcal{H})$ to $\mathcal{K}^{*}(\mathcal{H}')$ which will be useful.
\begin{definition}[Completely positive trace-preserving (CPTP) map] \label{def:CPTP}
A map $\Phi : \mathcal{K}^{*}(\mathcal{H}) \to \mathcal{K}^{*}(\mathcal{H}')$ is positive if it maps positive elements of $\mathcal{K}^{*}(\mathcal{H})$ to positive elements of $\mathcal{K}^{*}(\mathcal{H}')$. $\Phi$ is completely positive if when an ancilla of arbitrary finite dimension $n$ is coupled to the system, the induced map $I_n \otimes \Phi$, where $I_n$ is the identity map on the ancilla, is positive. $\Phi$ is trace-preserving if for any $A \in \mathcal{K}(\mathcal{H})$, $\mathrm{tr}[\Phi(A)] = \mathrm{tr}[A]$.
\end{definition}

\begin{theorem}[Stinespring representation theorem] \label{thm:stinespring} 
$T$ is a CPTP map between $\mathcal{K}^{*}(\mathcal{H})$ and $\mathcal{K}^{*}(\mathcal{H}')$ if and only if there exists a Hilbert space $\mathcal{H}_{E}$ and a isometric linear operator $V : \mathcal{H} \to \mathcal{H}' \otimes \mathcal{H}_{E}$ (i.e. $V^{*} V = \mathrm{Id}$) such that
\begin{align} \label{eq:ancilla_representation}
    T(A) = \mathrm{tr}_{E}(V A V^{*}),
\end{align}
where $\mathrm{tr}_{E}$ is the partial trace over the space $\mathcal{H}_{E}$. The space $\mathcal{H}_{E}$ can be chosen as $\mathcal{K}(\mathcal{H}')$.
\end{theorem}
The following result, which was shown by \cite{frank2013monotonicity} and \cite{beigi2013sandwiched}, shows that CPTP maps do not increase the quantum Rényi divergence between positive Hermitian operators.

\begin{lemma}[Data processing inequality the quantum Rényi divergence, \cite{frank2013monotonicity}, Thm. 1; \cite{beigi2013sandwiched}, Thm. 6] \label{lem:data_processing_QRD}
Let $\Phi : \mathcal{K}^{*}(\mathcal{H}) \to \mathcal{K}^{*}(\mathcal{H}')$ be a completely positive trace-preserving (CPTP) map (\autoref{def:CPTP}). Let $A, B$ be positive Hermitian operators in $\mathcal{K}^{*}(\mathcal{H})$ and let $\alpha \in [1/2,1) \cup (1,+\infty)$. Then,
\begin{align}
    D_{\alpha}(\Phi(A) || \Phi(B)) \leq D_{\alpha}(A || B)
\end{align}
\end{lemma}

\paragraph{Existence of CPTP maps between rank-1 operators.} Suppose that $\mathcal{H}, \mathcal{H}'$ are RKHS on $\mathcal{R}, \mathcal{R}'$ respectively, with feature maps $\phi : \mathcal{R} \to \mathcal{H}$, $\psi : \mathcal{R}' \to \mathcal{H}'$. Consider a mapping $g: \mathcal{R} \to \mathcal{R}'$. We are interested in finding a CPTP map $\Phi : \mathcal{K}^{*}(\mathcal{H}) \to \mathcal{K}^{*}(\mathcal{H}')$ such that
\begin{align}
    \forall x \in \mathcal{R}, \quad \Phi(\phi(x) \phi(x)^{*}) = \psi(g(x)) \psi(g(x))^{*}
\end{align}
The following lemma from \cite{chefles2004ontheexistence} can be used to prove the existence of such a map, in the case of $\mathcal{R}$ finite.
\begin{lemma}[\cite{chefles2004ontheexistence}, Thm. 2]
Let $(a_i a_i^{*})_i \subseteq \mathcal{K}^{*}(\mathcal{H})$ and $(b_i b_i^{*})_i \subseteq \mathcal{K}^{*}(\mathcal{H}')$ be finite families of rank-1 linear operators such that $\|a_i\| = 1$, $\|b_i\| = 1$, with equal cardinality. Let $G_a = {(G_a)}_{ij} = \langle a_i, a_j \rangle$ and $G_b = {(G_a)}_{ij} = \langle b_i, b_j \rangle$ be the Gram matrices. There exists a CPTP map that maps $a_i a_i^{*}$ to $b_i b_i^{*}$ if and only if $G_a = M \circ G_b$ for some matrix $M \geq 0$, where $\circ$ denotes the Hadamard (element-wise) product.
\end{lemma}
We can generalize this result to infinite families of rank-1 operators:
\begin{proposition}[Existence of CPTP map between rank-1 operators] \label{prop:CPTP_kernels}
Let $(a_x a_x^{*})_{x \in \mathcal{R}} \subseteq \mathcal{K}^{*}(\mathcal{H})$ and $(b_x b_x^{*})_{x\in \mathcal{R}} \subseteq \mathcal{K}^{*}(\mathcal{H}')$ be indexed families of rank-1 linear operators such that $\|a_x\| = 1$, $\|b_x\| = 1$. Let $k_a$, $k_b$ be kernel functions such that for all $x, x' \in \mathcal{R}$, $k_a(x,x') = \langle a_x, a_{x'} \rangle$ and $k_b(x,x') = \langle b_x, b_{x'} \rangle$. There exists a CPTP map that maps $a_x a_x^{*}$ to $b_x b_x^{*}$ for all $x \in \mathcal{R}$, if and only if $k_a = k_b \cdot \tilde{k}$ for some kernel function $\tilde{k}$.
\end{proposition}
\begin{proof}
Suppose that there exists a CPTP map $T$ such that $T(a_x a_x^{*}) = b_x b_x^{*}$ for all $x \in \mathcal{R}$. By the Stinespring representation theorem (\autoref{thm:stinespring}), this is equivalent to 
\begin{align} \label{eq:stinespring_consequence}
    \forall x \in \mathcal{R}, \quad \mathrm{tr}_{E}(V a_x a_x^{*} V^{*}) = b_x b_x^{*}, \quad \text{for some unitary } V : \mathcal{H} \to \mathcal{H}' \otimes \mathcal{H}_E. 
\end{align}
Note that $V a_x a_x^{*} V^{*} = V a_x (V a_x)^{*}$, If we let $(u_i)$ and $(v_i)$ be orthonormal bases of $\mathcal{H}'$ and $\mathcal{H}_E$, there exist coefficients $(c_{ij})$ such that
$V a_x = \sum_{i,j=1}^{+\infty} c_{ij} u_i \otimes v_j$. 
Hence, $V a_x (V a_x)^{*} = \sum_{i,j=1}^{+\infty} c_{ij} u_i \otimes v_j (\sum_{i',j'=1}^{+\infty} c_{i'j'} u_{i'}^{*} \otimes v_{j'}^{*}) = \sum_{i,j,i',j'=1}^{+\infty} c_{ij} c_{i'j'} u_i u_{i'}^{*} \otimes v_j v_{j'}^{*}$, which means that \begin{align} 
\mathrm{tr}_{E}(V a_x (V a_x)^{*}) = \sum_{i,i'=1}^{+\infty} \left(\sum_{j=1}^{+\infty} c_{ij} c_{i'j} \right) u_i u_{i'}^{*}.
\end{align}
This is equal to $b_x b_x^{*}$ if and only if $\sum_{j=1} c_{ij} c_{i'j} = \alpha_{i} \alpha_{i'}$ for all $i \in \mathbb{Z}^{+}$, in which case $b_x = \sum_{i=1}^{\infty} \alpha_i u_i$. \autoref{lem:rank_1} implies that this is equivalent to $c_{ij} = \beta_i \gamma_j$ for some sequences $(\beta_i)$ and $(\gamma_j)$, which holds iff $V a_x = (\sum_{i=1}^{+\infty} \beta_i u_i) \otimes (\sum_{j=1}^{+\infty} \gamma_j v_j)$, and $V a_x (V a_x)^{*} = (\sum_{i=1}^{+\infty} \beta_i u_i) (\sum_{i=1}^{+\infty} \beta_i u_i)^{*} \otimes (\sum_{j=1}^{+\infty} \gamma_j v_j) (\sum_{j=1}^{+\infty} \gamma_j v_j)^{*}$.

Note that the sequences $(\beta_i)$, $(\gamma_j)$ depend on the point $x$ by construction. Hence, defining $e_x = \sum_{j=1}^{+\infty} \gamma_j v_j$, we get that \eqref{eq:stinespring_consequence} is equivalent to
\begin{align} \label{eq:second_consequence}
    \forall x, x' \in \mathcal{R}, \quad V a_x (V a_{x'})^{*} = b_x b_{x'}^{*} \otimes e_x e_{x'}^{*}.
\end{align}
Note that $\mathrm{tr}(V a_x (V a_{x'})^{*}) = \langle V a_x, V a_{x'} \rangle = \langle a_x, V^{*} V a_{x'} \rangle = \langle a_x, a_{x'} \rangle$, while $\mathrm{tr}(b_x b_{x'}^{*} \otimes e_x e_{x'}^{*}) = \mathrm{tr}(b_x b_{x'}^{*}) \mathrm{tr}(e_x e_{x'}^{*}) = \langle b_x, b_{x'} \rangle \langle e_x, e_{x'} \rangle$. 
If we define the kernel $\tilde{k}(x,x') = \langle e_x, e_{x'} \rangle$, the implication from left to right follows. 

We still need to show \eqref{eq:second_consequence} from 
\begin{align} \label{eq:trace_equality}
\mathrm{tr}(a_x a_{x'}^{*}) = \mathrm{tr}((b_x \otimes e_x) (b_{x'}^{*} \otimes e_{x'}^{*}))
\end{align}
for the proof of the reverse implication. Equation \eqref{eq:trace_equality} implies that the map $a_x \mapsto b_x \otimes e_x$ is an isometry. Calling this map $V$, we get that $V a_x = b_x \otimes e_x$, and thus \eqref{eq:second_consequence} follows.
\end{proof}

\begin{lemma} \label{lem:rank_1}
Consider an indexed family $(c_{ij})_{i,j=1}^{+\infty}$ such that $\sum_{i,j=1}^{+\infty} c_{ij}^2 = 1$, and a sequence $(\alpha_i)_{i=1}^{+\infty}$ such that $\sum_{i=1}^{+\infty} \alpha_{i}^2 = 1$. The condition
\begin{align}
    \sum_{j=1}^{\infty} c_{ij} c_{i'j} = \alpha_{i} \alpha_{i'}
\end{align}
is equivalent to $\forall i,j \in \mathbb{Z}^{+}$, $c_{ij} = \beta_{i} \gamma_{j}$ for some sequences $(\beta_{i})_{i=1}^{+\infty}$, $(\gamma_{j})_{j=1}^{+\infty}$ that can be chosen such that $\sum_{i=1}^{+\infty} \beta_{i}^2 = 1$ and $\sum_{j=1}^{+\infty} \gamma_{j}^2 = 1$.
\end{lemma}
\begin{proof}
We will show that if we take any $i,i' \in \mathbb{Z}^{+}$ different from each other, there exist $\beta_{i}, \beta_{i'}$ such that $\forall j \in \mathbb{Z}^{+}$, $c_{ij} = \beta_{i} \gamma_{j}$ and $c_{i'j} = \beta_{i'} \gamma_{j}$.
Let $\R^{\infty}$ be the space of square-summable sequences.
Consider the linear function $C : \R^2 \to \R^{\infty}$ defined as $(v_i, v_{i'}) \mapsto v_{i} (c_{ij})_{j=1}^{+\infty} + v_{i'} (c_{i'j})_{j=1}^{+\infty}$. Its adjunct (or transpose) map is $C^{\top} : \R^{\infty} \to \R^2$ defined as $(u_{j})_{j=1}^{+\infty} \mapsto (\sum_{j=1}^{+\infty} u_{j} c_{ij}, \sum_{j=1}^{+\infty} u_{j} c_{i'j})$. Thus, the map $C^{\top} C : \R^2 \to \R^2$ is of the form
\begin{align}
    \begin{bmatrix}
    v_i \\
    v_{i'}
    \end{bmatrix} \mapsto
    \begin{bmatrix}
    \sum_{j=1}^{+\infty} (v_{i} c_{ij} + v_{i'} c_{i'j}) c_{ij} \\
    \sum_{j=1}^{+\infty} (v_{i} c_{ij} + v_{i'} c_{i'j}) c_{i'j}
    \end{bmatrix} &= 
    \begin{bmatrix}
    \sum_{j=1}^{+\infty} c_{ij} c_{ij} & \sum_{j=1}^{+\infty} c_{i'j} c_{ij} \\
    \sum_{j=1}^{+\infty} c_{ij} c_{i'j} & \sum_{j=1}^{+\infty} c_{i'j} c_{i'j}
    \end{bmatrix}
    \begin{bmatrix}
    v_i \\
    v_{i'}
    \end{bmatrix} \\ &= \begin{bmatrix}
    \alpha_{i} \\
    \alpha_{i'}
    \end{bmatrix}
    \begin{bmatrix}
    \alpha_{i} &
    \alpha_{i'}
    \end{bmatrix}
    \begin{bmatrix}
    v_i \\
    v_{i'}
    \end{bmatrix}
\end{align}
Hence, we obtain that $C^{\top} C$ has rank 1. Then, we apply the basic result that $\mathrm{rank}(C^{\top} C) = \mathrm{rank}(C)$ (the fact that $C$ maps to $\R^{\infty}$ is not a problem for the proof to go through, what matters is that the domain of $C$ is finite-dimensional), and we obtain that $\mathrm{rank}(C) = 1$. Thus, $(c_{ij})_{j=1}^{+\infty}$ and $(c_{i'j})_{j=1}^{+\infty}$ are proportional to each other, or equivalently, there exists $(\gamma_j)_{j=1}^{+\infty}$ such that $c_{ij} = \beta_{i} \gamma_{j}$ and $c_{i'j} = \beta_{i'} \gamma_{j}$. Using the condition $\sum_{i,j=1}^{+\infty} c_{ij}^2 = 1$ to show $\sum_{i=1}^{+\infty} \beta_{i}^2 = 1$ concludes the proof of the first implication.
The other implication is straightforward.
\end{proof}

\autoref{prop:CPTP_kernels} has an immediate corollary that we will use.


\begin{corollary} \label{cor:CPTP_kernel}
Consider a mapping $g: \mathcal{R} \to \mathcal{R}'$, and let $\mathcal{H}$, $\mathcal{H}'$ be RKHS over $\mathcal{R}$ and $\mathcal{R}'$, respectively, with feature functions $\phi : \mathcal{R} \to \mathcal{H}$, $\phi' : \mathcal{R}' \to \mathcal{H}'$ and kernel functions $k, k'$. We have that there exists a CPTP map $\Phi : \mathcal{K}^{*}(\mathcal{H}) \to \mathcal{K}^{*}(\mathcal{H}')$ such that
\begin{align}
    \forall x \in \mathcal{R}, \quad \Phi(\phi(x) \phi(x)^{*}) = \psi(g(x)) \psi(g(x))^{*}
\end{align}
if and only if there exists a kernel function $\tilde{k}$ such that
\begin{align} \label{eq:k_k'_k''_condition2}
    k(x,x') = k'(g(x),g(x')) \tilde{k}(x,x'), \quad \forall x, x' \in \mathcal{R}.
\end{align}
\end{corollary}

\begin{proof}
We apply \autoref{prop:CPTP_kernels} setting $a_x = \phi(x)$ and $b_x = \psi(g(x))$ for any $x \in \mathcal{R}$, and note that $\langle \psi(g(x)), \psi(g(x'))\rangle = k'(g(x),g(x'))$.
\end{proof}


\paragraph{Postprocessing property of RKRDP for deterministic $g$.} We proceed to show \autoref{thm:postprocessing_deterministic}.  

\textit{\textbf{Proof of \autoref{thm:postprocessing_deterministic}.}}
The theorem follows from the following statement: for any probability measures $p,q$ on $\mathcal{R}$, $D_{\alpha,\lambda}(\Sigma_{g_{\#} p}||\Sigma_{g_{\#} q}) \leq D_{\alpha,\lambda}(\Sigma_{p}||\Sigma_{q})$, where $g_{\#} p$ denotes the pushforward of $p$ by $g$. Indeed, if we set $p, q$ to be the laws $\mathcal{L}(f(D)),\mathcal{L}(f(D'))$ of the random variables $f(D)$, $f(D')$, respectively, we have that $g_{\#} p = \mathcal{L}(g(f(D)))$ and $g_{\#} q = \mathcal{L}(g(f(D')))$, and this implies equation \eqref{eq:contraction_postprocessing}.

To prove this statement, we use \autoref{lem:data_processing_QRD}. 
Applying \autoref{cor:CPTP_kernel} we obtain the existence of a CPTP map $\Phi : \mathcal{K}^{*}(\mathcal{H}) \to \mathcal{K}^{*}(\mathcal{H}')$ such that for any $x \in \mathcal{R}$, $\Phi(\phi(x) \phi(x)^{*}) = \psi(g(x)) \psi(g(x))^{*}$.
We reexpress $\Sigma_{p}$ and $\Sigma_{g_{\#} p}$ using their definition:
\begin{align} \label{eq:sigma_p_sigma_gp}
    \Sigma_{p} = \int_{\mathcal{R}} \phi(x) \phi(x)^{*} \, dp(x), \quad \Sigma_{g_{\#} p} = \int_{\mathcal{R}'} \psi(x) \psi(x)^{*} \, d(g_{\#} p)(x) = \int_{\mathcal{R}} \psi(g(x)) \psi(g(x))^{*} \, dp(x).
\end{align}
Hence, we obtain that 
\begin{align}
    \Phi(\Sigma_{p}) = \int_{\mathcal{R}} \Phi(\phi(x) \phi(x)^{*}) \, dp(x) = \int_{\mathcal{R}} \psi(g(x)) \psi(g(x))^{*} \, dp(x) = \Sigma_{g_{\#} p},
\end{align}
which concludes the proof.
\qed

The following corollary of \autoref{thm:postprocessing_deterministic} shows that when $\mathcal{H}$ is a tensor product of Hilbert spaces, the projection operator to one component satisfies the postprocessing inequality.
\begin{corollary}[Postprocessing property for projections]
Suppose that $\mathcal{H}_1, \mathcal{H}_2$ are RKHS on $\mathcal{R}_1$, $\mathcal{R}_2$ with kernels $k_1, k_2$ and we define $\mathcal{H} = \mathcal{H}_1 \otimes \mathcal{H}_2$, which has kernel $k((x,y),(x',y')) = k_1(x,x') k_2(y,y')$ by construction. Let $P_1 : \mathcal{R}_1 \times \mathcal{R}_2 \to \mathcal{R}_1$ be the projection operator to the space $\mathcal{R}_1$, i.e. $(x,y) \mapsto x$. If $f : \mathcal{D} \to \mathcal{R}_1 \times \mathcal{R}_2$ is a randomized mapping, then
\begin{align} \label{eq:contraction_postprocessing2}
    D_{\alpha,\lambda}(\Sigma_{\mathcal{L}(P_1(f(D)))}||\Sigma_{\mathcal{L}(P_1(f(D')))}) \leq D_{\alpha,\lambda}(\Sigma_{\mathcal{L}(f(D))}||\Sigma_{\mathcal{L}(f(D'))}).
\end{align}
Thus, if $f$ satisfies $(k,\alpha,\lambda,\varepsilon)$-RKRDP, then $P_1 \circ f$ satisfies $(k_1,\alpha,\lambda,\varepsilon)$-RKRDP.
\end{corollary}
\begin{proof}
The fact that $k((x,y),(y,y')) = k_1(x,x') k_2(y,y')$ implies that the condition \eqref{eq:k_k'_k''_condition} of \autoref{cor:CPTP_kernel} holds for $k$, $k_1$.
\end{proof}

\paragraph{Existence of CPTP map between rank-1 and generic operators.}
\begin{proposition}[Existence of CPTP map between rank-1 and generic operators] 
\label{prop:existence_CPTP_generic}
Let $(a_x a_x^{*})_{x \in \mathcal{R}} \subseteq \mathcal{K}^{*}(\mathcal{H})$ and $(b_x b_x^{*})_{x \in \mathcal{R}'} \subseteq \mathcal{K}^{*}(\mathcal{H}')$ be indexed families of rank-1 linear operators such that $\|a_x\| = 1$, $\|b_x\| = 1$. 
Let $\mathcal{G}$ be a class of functions from $\mathcal{R}$ to $\mathcal{R}'$ and let $g$ be a $\mathcal{G}$-valued random variable.
Let $k_a$, $k_b$ be kernel functions such that for all $x, x' \in \mathcal{R}$, $k_a(x,x') = \langle a_x, a_{x'} \rangle$ and for all $x, x' \in \mathcal{R}'$, $k_b(x,x') = \langle b_x, b_{x'} \rangle$. If there exists a family of kernels $(\tilde{k}_g)_{g}$ such that $k(x,x') = \mathbb{E}_{g}[\tilde{k}_g(x,x') k'(g(x),g(x'))]$, there exists a CPTP map that maps $a_x a_x^{*}$ to 
$\mathbb{E}_{g} [b_{g(x)} b_{g(x)}^{*}]$
for all $x \in \mathcal{R}$.
\end{proposition}
\begin{proof}
If there exists a family of kernels $(\tilde{k}_g)_{g}$ such that $k(x,x') = \mathbb{E}_{g}[\tilde{k}_g(x,x') k'(g(x),g(x'))]$. Denote by $\tilde{\mathcal{H}}_g$ 
the RKHS corresponding to $\tilde{k}_g$, 
and let $c_x^{g}$ be the feature map corresponding to $\tilde{k}_g$, evaluated at $x$. Then, 
\begin{align} \label{eq:isometry_expression}
\langle a_x, a_{x'} \rangle = \mathbb{E}_{g}[\langle c_x^{g}, c_{x'}^{g} \rangle \langle b_{g(x)}, b_{g(x')} \rangle ] = \mathbb{E}_{g}[\langle c_x^{g} \otimes b_{g(x)}, c_{x'}^{g} \otimes b_{g(x')} \rangle].
\end{align}
Define the map $V : \mathcal{H} \to \prod_{g \in \mathcal{G}} \tilde{\mathcal{H}}_g \otimes \mathcal{H}'_g$ as\footnote{Here, $\prod_{g \in \mathcal{G}}$ denotes the direct product of vector spaces indexed by $\mathcal{G}$. Direct products differ from direct sums $\oplus_{g \in \mathcal{G}}$ in that their elements do not have the restriction that all but finitely many coordinates must be zero. $\mathcal{H}'_g$ is just a copy of the space $\mathcal{H}'$.}
\begin{align}
    a_x \mapsto V a_x = c_x^{g} \otimes b_{g(x)}.
\end{align}
The space $\prod_{g \in \mathcal{G}} \tilde{\mathcal{H}}_g \otimes \mathcal{H}'_g$ is a Hilbert space\footnote{Formally, we have to deal with the completion of $\prod_{g \in \mathcal{G}} \tilde{\mathcal{H}}_g \otimes \mathcal{H}'_g$, as this space is not a priori complete.} when endowed with the product defined as $\langle c_x^{g} \otimes b_{x'}, c_{x''}^{g} \otimes b_{x'''} \rangle = \mathbb{E}_{g}[\langle c_x^{g} \otimes b_{x'}, c_{x''}^{g} \otimes b_{x'''} \rangle]$. With this construction, $V$ is an isometry from $\mathcal{H}$ to $\prod_{g \in \mathcal{G}} \tilde{\mathcal{H}}_g \otimes \mathcal{H}'_g$. We have that
\begin{align}
    V a_x (V a_x)^{*} &= \bigg(\prod_{g \in \mathcal{G}} c_x^{g} \otimes b_{g(x)} \bigg) \bigg(\prod_{g' \in \mathcal{G}} c_x^{g'} \otimes b_{g'(x)} \bigg)^{*} \in \mathcal{K}^{*} \bigg( \prod_{g \in \mathcal{G}} \tilde{\mathcal{H}}_g \otimes \mathcal{H}'_g \bigg) 
\end{align}
Via the inclusion $\mathcal{K}^{*} ( \prod_{g \in \mathcal{G}} \tilde{\mathcal{H}}_g \otimes \mathcal{H}'_g ) \subset \mathcal{K}^{*} ( (\prod_{g \in \mathcal{G}} \tilde{\mathcal{H}}_g) \otimes (\prod_{g \in \mathcal{G}} \mathcal{H}'_g) )$, we can take the partial trace with respect to the space $\prod_{g \in \mathcal{G}} \tilde{\mathcal{H}}_g$ to obtain
\begin{align}
    \mathrm{tr}_{\prod_{g \in \mathcal{G}} \tilde{\mathcal{H}}_g} (V a_x (V a_x)^{*}) = \bigg(\prod_{g \in \mathcal{G}} b_{g(x)} \bigg) \bigg(\prod_{g' \in \mathcal{G}} b_{g'(x)} \bigg)^{*} \in \mathcal{K}^{*} \bigg( \prod_{g \in \mathcal{G}} \mathcal{H}'_g \bigg)
\end{align}
Now we define the map $\mathcal{T} : \mathcal{K}^{*} ( \prod_{g \in \mathcal{G}} \mathcal{H}'_g ) \to \mathcal{K}^{*} (\mathcal{H}')$ 
as $\mathcal{T}((\prod_{g \in \mathcal{G}} h_g)(\prod_{g' \in \mathcal{G}} h'_{g})^{*}) = \mathbb{E}_{g} [ h_g (h'_g)^{*}]$, 
for arbitrary elements $h_g, h'_g \in \mathcal{H}'_g$. This is well defined because the spaces $\mathcal{H}'_g$ are copies of $\mathcal{H}'$. Note that $\mathcal{T}$ behaves formally like a partial trace, and that 
\begin{align}
    \mathcal{T} \bigg( \mathrm{tr}_{\prod_{g \in \mathcal{G}} \tilde{\mathcal{H}}_g} (V a_x  a_x^{*} V^{*}) \bigg) = \mathbb{E}_{g} [ b_{g(x)} b_{g(x)}^{*}]
\end{align}
If we define $\Phi : \mathcal{K}^{*}(\mathcal{H}) \to \mathcal{K}^{*}(\mathcal{H})$ as $\Phi(A) = \mathcal{T} \bigg( \mathrm{tr}_{\prod_{g \in \mathcal{G}} \tilde{\mathcal{H}}_g} (V A V^{*}) \bigg)$, we obtain that $\Phi$ maps $a_x a_x^{*}$ to
$\mathbb{E}_{g} [b_{g(x)} b_{g(x)}^{*}]$
for all $x \in \mathcal{R}$.
By the Stinespring representation theorem (\autoref{thm:stinespring}), $\Phi$ is a CPTP map, which concludes the proof.
\end{proof}

\textit{\textbf{Proof of \autoref{thm:postprocessing_randomized}.}}
    We apply \autoref{prop:existence_CPTP_generic} setting $a_x = \phi(x)$ for $x \in \mathcal{R}$ and $b_x = \psi(x)$ for $x \in \mathcal{R}'$. 
    We obtain the existence of a CPTP map $\Phi : \mathcal{K}^{*}(\mathcal{H}) \to \mathcal{K}^{*}(\mathcal{H}')$ such that for any $x \in \mathcal{R}$, $\Phi(\phi(x) \phi(x)^{*}) = \mathbb{E}_g[\psi(g(x)) \psi(g(x))^{*}]$.
    
    Then, we use \autoref{lem:data_processing_QRD} to get $D_{\alpha,\lambda}(\Phi(\Sigma_{\mathcal{L}(f(D))})||\Phi(\Sigma_{\mathcal{L}(f(D'))})) \leq D_{\alpha,\lambda}(\Sigma_{\mathcal{L}(f(D))}||\Sigma_{\mathcal{L}(f(D'))})$. We have that 
    \begin{align}
    \Sigma_{\mathcal{L}(f(D))} &= \int_{\mathcal{R}} \phi(x) \phi(x)^{*} \, d\mathcal{L}(f(D))(x), \\ \Sigma_{\mathcal{L}(g(f(D)))} &= \int_{\mathcal{R}'} \psi(x) \psi(x)^{*} \, d\mathcal{L}(g(f(D)))(x) = \int_{\mathcal{R}} \mathbb{E}_g[\psi(g(x)) \psi(g(x))^{*}] \, d\mathcal{L}(f(D))(x), 
    \end{align}
    and
    \begin{align}
        \Phi(\Sigma_{\mathcal{L}(f(D))}) = \int_{\mathcal{R}} \Phi(\phi(x) \phi(x)^{*}) \, d\mathcal{L}(f(D))(x) = \int_{\mathcal{R}} \mathbb{E}_g[\psi(g(x)) \psi(g(x))^{*}] \, d\mathcal{L}(f(D))(x),
    \end{align}
    Consequently, $\Phi(\Sigma_{\mathcal{L}(f(D))}) = \Sigma_{\mathcal{L}(g(f(D)))}$ and this concludes the proof.
\qed

\subsection{Results and proofs of the composition properties}

\textit{\textbf{Proof of \autoref{prop:parallel_composition}.}}
We let $\mathcal{H}_1$ be the RKHS on $\mathcal{R}_1$ with kernel function $k_1$ and feature map $\phi_1$. Similarly, we let $\mathcal{H}_2$ be the RKHS on $\mathcal{R}_2$ with kernel function $k_2$ and feature map $\phi_2$. If we define the tensor product $\mathcal{H} = \mathcal{H}_1 \otimes \mathcal{H}_2$, it is an RKHS with feature map $\phi : \mathcal{R}_1 \times \mathcal{R}_2 \to \mathcal{H}$ defined as $\phi(x,y) = \phi_1(x) \otimes \phi_2(y)$, and kernel $k((x,y),(x',y')) = \langle \phi_1(x) \otimes \phi_2(x'), \phi_1(y) \otimes \phi_2(y') \rangle = k_1(x,y) k_2(x',y')$. The map from $\mathcal{R}_1 \times \mathcal{R}_2$ to density operators on $\mathcal{H}$ is of the form $(x,y) \mapsto (\phi_1(x) \otimes \phi_2(y))(\phi_1(x)^{*} \otimes \phi_2(y)^{*}) = \phi_1(x)\phi_1(x)^{*} \otimes \phi_2(y)\phi_2(y)^{*}$. 
If $\mathcal{L}(f(D)), \mathcal{L}(g(D)), \mathcal{L}(f(D),g(D))$ are the laws of $f(D)$, $g(D)$ and $(f(D),g(D))$, we have
\begin{align}
\Sigma_{(f(D),g(D))} := \int_{\mathcal{R}_1} \int_{\mathcal{R}_2} \phi_1(x)\phi_1(x)^{*} \otimes \phi_2(y)\phi_2(y)^{*} \, d\mathcal{L}(g(D))(y) \, d\mathcal{L}(f(D))(x) = \Sigma_{f(D)} \otimes \Sigma_{g(D)}.
\end{align}
And then the additivity property of the regularized quantum Rényi divergence (\autoref{prop:prop_regularized}, Property 6) implies
\begin{align}
    D_{\alpha,\lambda}(\Sigma_{(f(D),g(D))} || \Sigma_{(f(D'),g(D'))}) \leq D_{\alpha,\lambda/3}(\Sigma_{f(D)} || \Sigma_{f(D')}) + D_{\alpha,\lambda/3}(\Sigma_{g(D)} || \Sigma_{g(D')}).
\end{align}
Thus, if $D_{\alpha,\lambda/3}(\Sigma_{f(D)} || \Sigma_{f(D')}) \leq \varepsilon_1$ and $D_{\alpha,\lambda/3}(\Sigma_{g(D)} || \Sigma_{g(D')}) \leq \varepsilon_2$, we obtain the inequality $D_{\alpha,\lambda}(\Sigma_{(f(D),g(D))} || \Sigma_{(f(D'),g(D'))}) \leq \varepsilon_1 + \varepsilon_2$.
\qed

\textit{\textbf{Proof of \autoref{prop:sequential_composition}.}}
We let $\mathcal{H}_1$ be an RKHS on $\mathcal{R}_1$ with kernel function $k_1$ and feature function $\phi_1$. Similarly, we let $\mathcal{H}_2$ be the RKHS on $\mathcal{R}_2$ with kernel function $k_2$ and feature function $\phi_2$. If we define the tensor product $\mathcal{H} = \mathcal{H}_1 \otimes \mathcal{H}_2$, it is an RKHS with feature function $\phi : \mathcal{R}_1 \times \mathcal{R}_2 \to \mathcal{H}$ defined as $\phi(x,x') = \phi_1(x) \otimes \phi_2(x')$, and kernel $k((x,x'),(y,y')) = \langle \phi_1(x) \otimes \phi_2(x'), \phi_1(y) \otimes \phi_2(y') \rangle = k_1(x,y) k_2(x',y')$. The map from $\mathcal{R}_1 \times \mathcal{R}_2$ to density operators on $\mathcal{H}$ is of the form $(x,x') \mapsto (\phi_1(x) \otimes \phi_2(x'))(\phi_1(x)^{*} \otimes \phi_2(x')^{*}) = \phi_1(x)\phi_1(x)^{*} \otimes \phi_2(x')\phi_2(x')^{*}$.

Then, by the definition of the regularized kernel Rényi divergence,
\begin{align}
    &\exp((\alpha-1)D_{\alpha,\lambda}(\Sigma_{\mathcal{L}(h(D))}||\Sigma_{\mathcal{L}(h(D'))})) \\ &= \mathrm{tr}\left[ \left( (\Sigma_{\mathcal{L}(h(D'))} + \lambda \mathrm{Id})^{\frac{1-\alpha}{2\alpha}} \Sigma_{\mathcal{L}(h(D))} (\Sigma_{\mathcal{L}(h(D'))} + \lambda \mathrm{Id})^{\frac{1-\alpha}{2\alpha}} \right)^{\alpha} \right].
\end{align}
Note that 
\begin{align}
\Sigma_{\mathcal{L}(h(D))} &= \int_{\mathcal{R}_1 \times \mathcal{R}_2} \phi_1(x)\phi_1(x)^{*} \otimes \phi_2(x')\phi_2(x')^{*} \, d(\mathcal{L}(h(D)))(x,x') \\ &= \int_{\mathcal{R}_1} \phi_1(x)\phi_1(x)^{*} \otimes \left( \int_{\mathcal{R}_2} \phi_2(x')\phi_2(x')^{*} \, d(\mathcal{L}(g(x,D)))(x') \right) \, d(\mathcal{L}(f(D)))(x) \\ &= \int_{\mathcal{R}_1} \phi_1(x)\phi_1(x)^{*} \otimes \Sigma_{\mathcal{L}(g(x,D))} \, d(\mathcal{L}(f(D)))(x) \\ &= \int_{\mathcal{R}_1} (\phi_1(x)\phi_1(x)^{*} \otimes \Sigma_{\mathcal{L}(g(x,D))}) \frac{d(\mathcal{L}(f(D)))}{d(\mathcal{L}(f(D')))}(x) \, d(\mathcal{L}(f(D')))(x)
\end{align}
where $\Sigma_{\mathcal{L}(g(x,D))} = \int_{\mathcal{R}_2} \phi_2(x')\phi_2(x')^{*} \, d(\mathcal{L}(g(x,D)))(x')$. Analogously, 
\begin{align}
\Sigma_{h(D')} = \int_{\mathcal{R}_1} \phi_1(x)\phi_1(x)^{*} \otimes \Sigma_{\mathcal{L}(g(x,D'))} \, d(\mathcal{L}(f(D')))(x).    
\end{align}

Applying \autoref{lem:frank_lieb_lemma}(i), we obtain that
\begin{align} 
\begin{split} \label{eq:tr_adaptive_composition}
    \mathrm{tr} &\left[\left( (\Sigma_{\mathcal{L}(h(D'))}+\lambda \mathrm{Id})^{\frac{1-\alpha}{2\alpha}} \Sigma_{\mathcal{L}(h(D))} (\Sigma_{\mathcal{L}(h(D'))} + \lambda \mathrm{Id})^{\frac{1-\alpha}{2\alpha}} \right)^{\alpha} \right] \\ \leq 
    \int_{\mathcal{R}_1} \mathrm{tr}\big[ \big( &(\phi_1(x)\phi_1(x)^{*} \otimes \Sigma_{\mathcal{L}(g(x,D'))} +\lambda \mathrm{Id})^{\frac{1-\alpha}{2\alpha}} (\phi_1(x)\phi_1(x)^{*} \otimes \Sigma_{\mathcal{L}(g(x,D))}) \cdot \\ &\frac{d(\mathcal{L}(f(D)))}{d(\mathcal{L}(f(D')))}(x) (\phi_1(x)\phi_1(x)^{*} \otimes \Sigma_{\mathcal{L}(g(x,D'))} + \lambda \mathrm{Id})^{\frac{1-\alpha}{2\alpha}} \big)^{\alpha} \big] \, d(\mathcal{L}(f(D')))(x)
\end{split}
\end{align}
Note that 
\begin{align} 
\begin{split} \label{eq:adaptive_2}
\mathrm{tr}\bigg[ \big( &(\phi_1(x)\phi_1(x)^{*} \otimes \Sigma_{\mathcal{L}(g(x,D'))} + \lambda \mathrm{Id})^{\frac{1-\alpha}{2\alpha}} (\phi_1(x)\phi_1(x)^{*} \otimes \Sigma_{\mathcal{L}(g(x,D))}) \cdot \\ &(\phi_1(x)\phi_1(x)^{*} \otimes \Sigma_{\mathcal{L}(g(x,D'))} + \lambda \mathrm{Id})^{\frac{1-\alpha}{2\alpha}} \big)^{\alpha} \bigg] \\ = \exp &((\alpha-1) D_{\alpha}(\phi_1(x)\phi_1(x)^{*} \otimes \Sigma_{\mathcal{L}(g(x,D))}||\phi_1(x)\phi_1(x)^{*} \otimes \Sigma_{\mathcal{L}(g(x,D'))} + \lambda \mathrm{Id})) \\ \leq \exp &((\alpha-1) D_{\alpha}(\phi_1(x)\phi_1(x)^{*} \otimes \Sigma_{\mathcal{L}(g(x,D))}||\phi_1(x)\phi_1(x)^{*} \otimes (\Sigma_{\mathcal{L}(g(x,D'))} + \lambda \mathrm{Id}))) \\ =
\mathrm{tr}\bigg[ \big( &(\phi_1(x)\phi_1(x)^{*} \otimes (\Sigma_{\mathcal{L}(g(x,D'))} + \lambda \mathrm{Id}))^{\frac{1-\alpha}{2\alpha}} (\phi_1(x)\phi_1(x)^{*} \otimes \Sigma_{\mathcal{L}(g(x,D))}) \cdot \\ &(\phi_1(x)\phi_1(x)^{*} \otimes (\Sigma_{\mathcal{L}(g(x,D'))} + \lambda \mathrm{Id}))^{\frac{1-\alpha}{2\alpha}} \big)^{\alpha} \bigg]
\end{split}
\end{align}
where the inequality holds by \autoref{prop:4_lennert} using that $\phi_1(x)\phi_1(x)^{*} \otimes (\Sigma_{\mathcal{L}(g(x,D'))} + \lambda \mathrm{Id}_2) = \phi_1(x)\phi_1(x)^{*} \otimes \Sigma_{\mathcal{L}(g(x,D'))} + \lambda \phi_1(x)\phi_1(x)^{*} \otimes \mathrm{Id}_2 \leq \phi_1(x)\phi_1(x)^{*} \otimes \Sigma_{\mathcal{L}(g(x,D'))} + \lambda \mathrm{Id}_1 \otimes \mathrm{Id}_2$, and because $\alpha > 1$. We develop the right-hand side:  
\begin{align} 
\begin{split} \label{eq:adaptive_3}
    \mathrm{tr}\bigg[ \big( &(\phi_1(x)\phi_1(x)^{*} \otimes (\Sigma_{\mathcal{L}(g(x,D'))} + \lambda \mathrm{Id}))^{\frac{1-\alpha}{2\alpha}} (\phi_1(x)\phi_1(x)^{*} \otimes \Sigma_{\mathcal{L}(g(x,D))}) \cdot \\ &(\phi_1(x)\phi_1(x)^{*} \otimes (\Sigma_{\mathcal{L}(g(x,D'))} + \lambda \mathrm{Id}))^{\frac{1-\alpha}{2\alpha}} \big)^{\alpha} \bigg]
    \\ = \mathrm{tr}\bigg[ \big( &(\phi_1(x)\phi_1(x)^{*})^{\frac{1-\alpha}{2\alpha}} \phi_1(x)\phi_1(x)^{*} (\phi_1(x)\phi_1(x)^{*})^{\frac{1-\alpha}{2\alpha}} \big)^{\alpha} \bigg] \cdot \\ &\mathrm{tr}\bigg[ \big(  (\Sigma_{\mathcal{L}(g(x,D'))} + \lambda \mathrm{Id})^{\frac{1-\alpha}{2\alpha}} \Sigma_{g(x,D)} (\Sigma_{\mathcal{L}(g(x,D'))} + \lambda \mathrm{Id})^{\frac{1-\alpha}{2\alpha}} \big)^{\alpha} \bigg] 
    \\ = |\langle \phi_1 &(x), \phi_1(x) \rangle|^{2\alpha} \mathrm{tr}\bigg[ \big(  (\Sigma_{\mathcal{L}(g(x,D'))} + \lambda \mathrm{Id})^{\frac{1-\alpha}{2\alpha}} \Sigma_{\mathcal{L}(g(x,D))} (\Sigma_{\mathcal{L}(g(x,D'))} + \lambda \mathrm{Id})^{\frac{1-\alpha}{2\alpha}} \big)^{\alpha} \bigg] \\ = \mathrm{tr}\bigg[ \big(  &(\Sigma_{\mathcal{L}(g(x,D'))} + \lambda \mathrm{Id})^{\frac{1-\alpha}{2\alpha}} \Sigma_{\mathcal{L}(g(x,D))} (\Sigma_{\mathcal{L}(g(x,D'))} + \lambda \mathrm{Id})^{\frac{1-\alpha}{2\alpha}} \big)^{\alpha} \bigg],
\end{split}
\end{align}
Using \eqref{eq:adaptive_2} and \eqref{eq:adaptive_3},
we upper-bound the right-hand side of equation \eqref{eq:tr_adaptive_composition} by
\begin{align}
    \int_{\mathcal{R}_1} &\mathrm{tr}\bigg[ \big(  (\Sigma_{\mathcal{L}(g(x,D'))} + \lambda \mathrm{Id})^{\frac{1-\alpha}{2\alpha}} \Sigma_{\mathcal{L}(g(x,D))} (\Sigma_{\mathcal{L}(g(x,D'))} + \lambda \mathrm{Id})^{\frac{1-\alpha}{2\alpha}} \big)^{\alpha} \bigg] \\ &\bigg(\frac{d(\mathcal{L}(f(D)))}{d(\mathcal{L}(f(D')))}(x) \bigg)^{\alpha} \, d(\mathcal{L}(f(D')))(x) \\ = \int_{\mathcal{R}_1} &\exp((\alpha-1)D_{\alpha}(\Sigma_{\mathcal{L}(g(x,D))}||\Sigma_{\mathcal{L}(g(x,D'))} + \lambda \mathrm{Id})) \bigg(\frac{d(\mathcal{L}(f(D)))}{d(\mathcal{L}(f(D')))}(x) \bigg)^{\alpha} \, d(\mathcal{L}(f(D')))(x) \\ \leq 
    \int_{\mathcal{R}_1} &\exp((\alpha-1)\varepsilon_2) \bigg(\frac{d(\mathcal{L}(f(D)))}{d(\mathcal{L}(f(D')))}(x) \bigg)^{\alpha} \, d(\mathcal{L}(f(D')))(x) \\ = \exp &((\alpha-1) \varepsilon_2) \exp((\alpha -1) D_{\alpha}(f(D)||f(D'))) \leq \exp((\alpha-1) (\varepsilon_1 + \varepsilon_2)).
\end{align}
Thus, we get that we obtain that $\exp((\alpha-1)D_{\alpha}(\Sigma_{\mathcal{L}(h(D))}||\Sigma_{\mathcal{L}(h(D'))})) \leq \exp((\alpha-1) (\varepsilon_1 + \varepsilon_2))$. Since $x \mapsto \exp((\alpha-1)x)$ is an increasing function because $\alpha > 1$, we conclude that $D_{\alpha}(\Sigma_{\mathcal{L}(h(D))}||\Sigma_{\mathcal{L}(h(D'))}) \leq \varepsilon_1 + \varepsilon_2$.
\qed

\subsection{Proof of \autoref{prop:lower_bound_renyi}} \label{subsec:proof_lower_bound}

We deferred the proof of \autoref{prop:lower_bound_renyi} because we need to make use of the data processing inequality introduced earlier in this section.

First, note that the map $\Phi : \mathcal{K}^{*}(\mathcal{H}) \to \mathcal{K}^{*}(\prod_{x \in \mathcal{X}} \mathcal{H}_{x})$
defined as\footnote{Here, for any $x \in \mathcal{X}$, $\mathcal{H}_{x}$ is a copy of $\mathcal{H}$.}
\begin{align}
    \Phi(A) = 
    (\phi(x) \phi(x)^{*} \Sigma^{-1/2} A \Sigma^{-1/2} \phi(x) \phi(x)^{*})_{x \in \mathcal{X}} 
\end{align}
is CPTP. Here, we define the trace of $A = (A_x)_{x \in \mathcal{X}} \in \mathcal{K}^{*}(\prod_{x \in \mathcal{X}} \mathcal{H}_{x})$ as $\mathrm{tr}(A) = \int_{\mathcal{X}} \mathrm{tr}(A_x) \, d\tau(x)$. Then,
\begin{align}
    \mathrm{tr}(\Phi(A)) &= \int_{\mathcal{X}} \mathrm{tr}(\phi(x) \phi(x)^{*} \Sigma^{-1/2} A \Sigma^{-1/2} \phi(x) \phi(x)^{*}) \, d\tau(x) \\ &= \int_{\mathcal{X}} \mathrm{tr}(\Sigma^{-1/2} A \Sigma^{-1/2} \phi(x) \phi(x)^{*}) \, d\tau(x) \\ &= \mathrm{tr}(\Sigma^{-1/2} A \Sigma^{-1/2} \Sigma) = \mathrm{tr}(A),
\end{align}
which means that $\Phi$ is trace-preserving. It is also completely positive because its components $\Phi_x = \phi(x) \phi(x)^{*} \Sigma^{-1/2} A \Sigma^{-1/2} \phi(x) \phi(x)^{*}$ are completely positive.
Hence, by \autoref{lem:data_processing_QRD},
\begin{align}
    D_{\alpha}(\Phi(\Sigma_p)||\Phi(\Sigma_q)) \leq D_{\alpha}(\Sigma_p||\Sigma_q).
\end{align}
Note that 
\begin{align} \label{eq:d_alpha_phi}
    D_{\alpha}(\Phi(\Sigma_p)||\Phi(\Sigma_q)) &= \frac{1}{\alpha-1} \log \bigg( \mathrm{tr}\left[\left( \Phi(\Sigma_q)^{\frac{1-\alpha}{2\alpha}} \Phi(\Sigma_p) \Phi(\Sigma_q)^{\frac{1-\alpha}{2\alpha}} \right)^\alpha \right] \bigg) \\ &= \frac{1}{\alpha-1} \log \bigg( \int_{\mathcal{Y}} \mathrm{tr}\left[ \left( \Phi_x(\Sigma_q)^{\frac{1-\alpha}{2\alpha}} \Phi_x(\Sigma_p) \Phi_x(\Sigma_q)^{\frac{1-\alpha}{2\alpha}} \right)^\alpha \right] \, d\tau(x) \bigg)
\end{align}
And
\begin{align}
    &\mathrm{tr}\left[ \left( \Phi_x(\Sigma_q)^{\frac{1-\alpha}{2\alpha}} \Phi_x(\Sigma_p) \Phi_x(\Sigma_q)^{\frac{1-\alpha}{2\alpha}} \right)^\alpha \right] \\ &= \mathrm{tr}\bigg[ \bigg( \left(\phi(x) \phi(x)^{*} \Sigma^{-1/2} \Sigma_q \Sigma^{-1/2} \phi(x) \phi(x)^{*} \right)^{\frac{1-\alpha}{2\alpha}} \left(\phi(x) \phi(x)^{*} \Sigma^{-1/2} \Sigma_p \Sigma^{-1/2} \phi(x) \phi(x)^{*} \right) \\ &\left(\phi(x) \phi(x)^{*} \Sigma^{-1/2} \Sigma_q \Sigma^{-1/2} \phi(x) \phi(x)^{*} \right)^{\frac{1-\alpha}{2\alpha}} \bigg)^\alpha \bigg]
    \\ = &\langle \phi(x), \Sigma^{-1/2} \Sigma_q \Sigma^{-1/2} \phi(x) \rangle^{1-\alpha} \langle \phi(x), \Sigma^{-1/2} \Sigma_p \Sigma^{-1/2} \phi(x) \rangle^{\alpha} \cdot \\ &\mathrm{tr}\bigg[ \bigg( \left(\phi(x) \phi(x)^{*} \right)^{\frac{1-\alpha}{2\alpha}} \left(\phi(x) \phi(x)^{*} \right) \left(\phi(x) \phi(x)^{*} \right)^{\frac{1-\alpha}{2\alpha}} \bigg)^\alpha \bigg] \\ = &\langle \Sigma^{-1/2} \phi(x), \Sigma_q \Sigma^{-1/2} \phi(x) \rangle^{1-\alpha} \langle \Sigma^{-1/2} \phi(x), \Sigma_p \Sigma^{-1/2} \phi(x) \rangle^{\alpha} \\ = &\mathrm{tr}(\Sigma_q \Sigma^{-1/2} \phi(x) \phi(x)^{*} \Sigma^{-1/2})^{1-\alpha} \mathrm{tr}(\Sigma_p \Sigma^{-1/2} \phi(x) \phi(x)^{*} \Sigma^{-1/2})^{\alpha} = \mathrm{tr}(\Sigma_q D(x))^{1-\alpha} \mathrm{tr}(\Sigma_p D(x))^{\alpha}.
\end{align}
We have that
\begin{align}
    \mathrm{tr}(D(y) \Sigma_p) &= \mathrm{tr} \bigg( \int_{\mathcal{X}} \Sigma^{-1/2} \phi(y) \phi(y)^{*} \Sigma^{-1/2} \phi(x) \phi(x)^{*} \, dp(x) \bigg) \\ &= \int_{\mathcal{X}} \mathrm{tr} \bigg( \Sigma^{-1/2} \phi(y) \phi(y)^{*} \Sigma^{-1/2} \phi(x) \phi(x)^{*} \bigg) \, dp(x) \\ &=  \int_{\mathcal{X}} \mathrm{tr} \bigg( \phi(x)^{*} \Sigma^{-1/2} \phi(y) \phi(y)^{*} \Sigma^{-1/2} \phi(x) \bigg) \, dp(x) \\ &= \int_{\mathcal{X}} \langle \phi(x), \Sigma^{-1/2} \phi(y) \rangle^2 \, dp(x) = \int_{\mathcal{X}} h(x,y) \, dp(x),
\end{align}
defining the function $h(x,y) = \langle \phi(x), \Sigma^{-1/2} \phi(y) \rangle^2$, which fulfills
\begin{align}
    \forall y \in \mathcal{X}, \, \int_{\mathcal{X}} h(x,y) \, d\tau(x) = \bigg\langle \phi(y), \Sigma^{-1/2} \bigg( \int_{\mathcal{X}} \phi(x) \phi(x)^{*} \, d\tau(x) \bigg) \Sigma^{-1/2} \phi(y) \bigg\rangle = \langle \phi(y), \phi(y) \rangle = 1.
\end{align}
Plugging everything back into \eqref{eq:d_alpha_phi}, we obtain
\begin{align}
    D_{\alpha}(\Phi(\Sigma_p)||\Phi(\Sigma_q)) &= \frac{1}{\alpha-1} \log \bigg( \int_{\mathcal{Y}} \bigg(\int_{\mathcal{X}} h(x,y) \, dp(x)\bigg)^{\alpha} \bigg(\int_{\mathcal{X}} h(x,y) \, dq(x)\bigg)^{1-\alpha} \, d\tau(y) \bigg) \\ &= \frac{1}{\alpha-1} \log \bigg( \int_{\mathcal{Y}} \tilde{p}(y)^{\alpha} \tilde{q}(y)^{1-\alpha} \, d\tau(y) \bigg) = D_{\alpha}(\tilde{p}||\tilde{q}).
\end{align}

\section{Proofs of \autoref{sec:estimating}} \label{sec:proofs_estimating}

\textit{\textbf{Proof of \autoref{thm:estimating}.}}
If $\hat{\Sigma}_p, \hat{\Sigma}_q$ are positive Hermitian operators with unit trace such that $\|\hat{\Sigma}_p - \Sigma_p\| \leq t, \|\hat{\Sigma}_q - \Sigma_q\| \leq t$, where $t \leq \lambda := \delta e^{-\varepsilon}$, we have
\begin{align} 
\begin{split}\label{eq:log_tr_bound}
    \bigg| \frac{1}{\alpha-1} \log \bigg( \mathrm{tr} &\left[\left( (\hat{\Sigma}_q + \lambda \mathrm{Id})^{\frac{1-\alpha}{2\alpha}} \hat{\Sigma}_p (\hat{\Sigma}_q + \lambda \mathrm{Id})^{\frac{1-\alpha}{2\alpha}} \right)^\alpha \right] \bigg) \\ - \frac{1}{\alpha-1} \log \bigg( &\mathrm{tr}\left[\left( (\Sigma_q + \lambda \mathrm{Id})^{\frac{1-\alpha}{2\alpha}} \Sigma_p (\Sigma_q + \lambda \mathrm{Id})^{\frac{1-\alpha}{2\alpha}} \right)^\alpha \right] \bigg) \bigg| \\ \leq \frac{(\sigma_{\mathrm{max}}(\Sigma_q) + t + \lambda)^{\alpha-1}}{(\alpha-1)\mathrm{tr}[\Sigma_p^{\alpha}]} \bigg| &\mathrm{tr}\bigg[\left( (\hat{\Sigma}_q + \lambda \mathrm{Id})^{\frac{1-\alpha}{2\alpha}} \hat{\Sigma}_p (\hat{\Sigma}_q + \lambda \mathrm{Id})^{\frac{1-\alpha}{2\alpha}} \right)^\alpha \\ &- \left( (\Sigma_q + \lambda \mathrm{Id})^{\frac{1-\alpha}{2\alpha}} \Sigma_p (\Sigma_q + \lambda \mathrm{Id})^{\frac{1-\alpha}{2\alpha}} \right)^\alpha \bigg] \bigg|,
\end{split}
\end{align}
where we used that $\mathrm{tr}[( (\sigma_{\mathrm{max}}(\Sigma_q) + t + \lambda)^{\frac{1-\alpha}{\alpha}} \Sigma_p )^\alpha ] \geq (\sigma_{\mathrm{max}}(\Sigma_q) + t + \lambda)^{1-\alpha} \mathrm{tr}[\Sigma_p^{\alpha}] = (\|\Sigma_q\| + t + \lambda)^{1-\alpha} \mathrm{tr}[\Sigma_p^{\alpha}]$.
To bound the right-hand side, we write
\begin{align} 
\begin{split} \label{eq:tr_bound}
    &\bigg| \mathrm{tr}\left[\left( (\hat{\Sigma}_q + \lambda \mathrm{Id})^{\frac{1-\alpha}{2\alpha}} \hat{\Sigma}_p (\hat{\Sigma}_q + \lambda \mathrm{Id})^{\frac{1-\alpha}{2\alpha}} \right)^\alpha - \left( (\Sigma_q + \lambda \mathrm{Id})^{\frac{1-\alpha}{2\alpha}} \Sigma_p (\Sigma_q + \lambda \mathrm{Id})^{\frac{1-\alpha}{2\alpha}} \right)^\alpha \right] \bigg| \\ &\leq \bigg| \mathrm{tr}\left[\left( (\hat{\Sigma}_q + \lambda \mathrm{Id})^{\frac{1-\alpha}{2\alpha}} \hat{\Sigma}_p (\hat{\Sigma}_q + \lambda \mathrm{Id})^{\frac{1-\alpha}{2\alpha}} \right)^\alpha - \left( (\hat{\Sigma}_q + \lambda \mathrm{Id})^{\frac{1-\alpha}{2\alpha}} \Sigma_p (\hat{\Sigma}_q + \lambda \mathrm{Id})^{\frac{1-\alpha}{2\alpha}} \right)^\alpha \right] \bigg| \\ &+ \bigg| \mathrm{tr}\left[\left( (\hat{\Sigma}_q + \lambda \mathrm{Id})^{\frac{1-\alpha}{2\alpha}} \Sigma_p (\hat{\Sigma}_q + \lambda \mathrm{Id})^{\frac{1-\alpha}{2\alpha}} \right)^\alpha - \left( (\Sigma_q + \lambda \mathrm{Id})^{\frac{1-\alpha}{2\alpha}} \Sigma_p (\Sigma_q + \lambda \mathrm{Id})^{\frac{1-\alpha}{2\alpha}} \right)^\alpha \right] \bigg|
\end{split}
\end{align}
We bound the first term:
\begin{align}
    &\bigg| \mathrm{tr}\left[\left( (\hat{\Sigma}_q + \lambda \mathrm{Id})^{\frac{1-\alpha}{2\alpha}} \hat{\Sigma}_p (\hat{\Sigma}_q + \lambda \mathrm{Id})^{\frac{1-\alpha}{2\alpha}} \right)^\alpha - \left( (\hat{\Sigma}_q + \lambda \mathrm{Id})^{\frac{1-\alpha}{2\alpha}} \Sigma_p (\hat{\Sigma}_q + \lambda \mathrm{Id})^{\frac{1-\alpha}{2\alpha}} \right)^\alpha \right] \bigg| \\ &\leq \lambda^{1-\alpha} \bigg| \mathrm{tr}\left[ \hat{\Sigma}_p^\alpha - \Sigma_p^\alpha \right] \bigg|
\end{align}
and we have that 
\begin{align}
&\bigg|\mathrm{tr}\left[ \hat{\Sigma}_p^\alpha - \Sigma_p^\alpha \right]\bigg| = \bigg|\sum_{n \geq 0} (\sigma_n(\hat{\Sigma}_p)^\alpha - \sigma_n(\Sigma_p)^\alpha) \bigg| \\ &= \alpha \sum_{n \geq 0} |\sigma_n(\hat{\Sigma}_p) - \sigma_n(\Sigma_p) | \max\{ \sigma_n(\hat{\Sigma}_p)^{\alpha-1}, \sigma_n(\Sigma_p)^{\alpha-1} \} \leq \alpha \|\hat{\Sigma}_p - \Sigma_p\| \bigg(\sum_{n \geq 0} \sigma_n(\hat{\Sigma}_p) + \sigma_n(\Sigma_p) \bigg) \\ &= 2 \alpha \|\hat{\Sigma}_p - \Sigma_p\| \leq 2 \alpha t.
\end{align}
where we used that $\alpha \geq 2$.
We bound the second term of \eqref{eq:tr_bound}:
\begin{align}
    &\bigg| \mathrm{tr}\left[\left( (\hat{\Sigma}_q + \lambda \mathrm{Id})^{\frac{1-\alpha}{2\alpha}} \Sigma_p (\hat{\Sigma}_q + \lambda \mathrm{Id})^{\frac{1-\alpha}{2\alpha}} \right)^\alpha - \left( (\Sigma_q + \lambda \mathrm{Id})^{\frac{1-\alpha}{2\alpha}} \Sigma_p (\Sigma_q + \lambda \mathrm{Id})^{\frac{1-\alpha}{2\alpha}} \right)^\alpha \right] \bigg| \\ &\leq \bigg| \mathrm{tr}\left[\left( (\Sigma_q + (\lambda -t) \mathrm{Id})^{\frac{1-\alpha}{2\alpha}} \Sigma_p (\Sigma_q + (\lambda -t) \mathrm{Id})^{\frac{1-\alpha}{2\alpha}} \right)^\alpha - \left( (\Sigma_q + \lambda \mathrm{Id})^{\frac{1-\alpha}{2\alpha}} \Sigma_p (\Sigma_q + \lambda \mathrm{Id})^{\frac{1-\alpha}{2\alpha}} \right)^\alpha \right] \bigg| \\ &\leq \bigg| \mathrm{tr}\left[((\lambda -t)^{1-\alpha} - \lambda^{1-\alpha}) \Sigma_p^{\alpha} \right] \bigg| \leq (\lambda -t)^{1-\alpha} - \lambda^{1-\alpha}.
\end{align}
In the first inequality we used that $\Sigma_q - t \mathrm{Id} \leq \hat{\Sigma}_q$ and that the exponent $\frac{1-\alpha}{2\alpha}$ is negative. Hence, the right-hand side of \eqref{eq:tr_bound} is upper-bounded by $2 \alpha \lambda^{1-\alpha}t + (\lambda -t)^{1-\alpha} - \lambda^{1-\alpha}$, and the right-hand side of \eqref{eq:log_tr_bound} is upper-bounded by 
\begin{align} \label{eq:tr_upper_bound_2}
    \frac{(\|\Sigma_q\| + t + \lambda)^{\alpha-1}}{(\alpha-1)\mathrm{tr}[\Sigma_p^{\alpha}]} (2 \alpha \lambda^{1-\alpha} t + (\lambda -t)^{1-\alpha} - \lambda^{1-\alpha})
\end{align}
Suppose that $t = \tilde{t} \lambda$ with $\tilde{t} \leq \frac{1}{\alpha}$. Then, $(\lambda -t)^{1-\alpha} - \lambda^{1-\alpha} = \lambda^{1-\alpha} (1-\tilde{t})^{1-\alpha} - \lambda^{1-\alpha} = \lambda^{1-\alpha} ((1-\tilde{t})^{1-\alpha} -1) \leq (\alpha - 1) (1-\tilde{t})^{-\alpha} t \leq 4 (\alpha - 1) t$, 
where we used that for all $\alpha \geq 2$, $(1-\frac{1}{\alpha})^{-\alpha} \leq 4$.
Hence, equation \eqref{eq:tr_upper_bound_2} is upper-bounded by
\begin{align}
    \frac{(\|\Sigma_q\| + (1+\frac{1}{\alpha}) \lambda)^{\alpha-1}}{(\alpha-1)\mathrm{tr}[\Sigma_p^{\alpha}]} (2 \alpha \lambda^{1-\alpha} + 4 (\alpha - 1)) t.
\end{align}
Now, we apply \autoref{thm:bernstein_operator} as follows: if ${(x_i)}_{i=1}^{n}$ are samples from $p$ and ${(y_i)}_{i=1}^{n}$ are samples from $p$, we let $X_i = \frac{1}{n}(\phi(x_i) \phi(x_i)^{*} - \Sigma_p)$ and $Y_i = \frac{1}{n}(\phi(y_i) \phi(y_i)^{*} - \Sigma_q)$. Note that $\mathbb{E}[X_i] = 0$, $\mathbb{E}[Y_i] = 0$, and that  $\|\sum_{i=1}^{n} \mathbb{E} X_i^2\| = \|\frac{1}{n^2} \sum_{i=1}^{n} \mathbb{E} [\phi(y_i) \phi(y_i)^{*} - \phi(y_i) \phi(y_i)^{*} \Sigma_p - \Sigma_p \phi(y_i) \phi(y_i)^{*} + \Sigma_p^2]\| = \frac{1}{n} \|\Sigma_p - \Sigma_p^2\|$, while $\|\sum_{i=1}^{n} \mathbb{E} Y_i^2\| = \frac{1}{n} \|\Sigma_q - \Sigma_q^2\|$. Also, $\|X_i\| = \frac{1}{n}\|\phi(x_i) \phi(x_i)^{*} - \Sigma_p\| \leq \frac{1}{n}$ and $\|X_i\| \leq \frac{1}{n}$. Thus, for any 
$t \geq \frac{1}{6}(\frac{1}{n} + \sqrt{\frac{1}{n^2} + \frac{36}{n} \|\Sigma_p - \Sigma_p^2\|})$,
\begin{align}
    \mathrm{Pr}\bigg(\bigg\|\Sigma_p - \frac{1}{n} \sum_{i=1}^{n} \phi(x_i) \phi(x_i)^{*}\bigg\| > t \bigg) \leq 14 \frac{\mathrm{Tr}(\Sigma_p - \Sigma_p^2)}{\|\Sigma_p - \Sigma_p^2\|} \exp \bigg(-\frac{n t^2/2}{\|\Sigma_p - \Sigma_p^2\| + t/3} \bigg),
\end{align}
where we used that the effective rank is $r(\sum_{i=1}^{n} \mathbb{E} X_i^2) = r(\frac{1}{n}(\Sigma_p - \Sigma_p^2)) = \frac{\mathrm{Tr}(\Sigma_p - \Sigma_p^2)}{\|\Sigma_p - \Sigma_p^2\|}$. The analogous result holds for $\Sigma_q$. If we set the right-hand side equal to $x_0$, we get
\begin{align}
    &\frac{n t^2/2}{\|\Sigma_p - \Sigma_p^2\| + t/3} = \log \bigg( \frac{14 \mathrm{Tr}(\Sigma_p - \Sigma_p^2)}{\|\Sigma_p - \Sigma_p^2\| x_0} \bigg) \implies \frac{n}{2} t^2 - \frac{\ell}{3} t - \ell \|\Sigma_p - \Sigma_p^2\| = 0 \\ &\implies t = \frac{\frac{\ell}{3} + \sqrt{(\frac{\ell}{3})^2 + 2 n \ell \|\Sigma_p - \Sigma_p^2\|}}{n},
\end{align}
where we defined $\ell = \log ( 14 \mathrm{Tr}(\Sigma_p - \Sigma_p^2)/(\|\Sigma_p - \Sigma_p^2\| x_0) )$.
\qed

\begin{lemma}[Bernstein inequality for Hermitian operators, \cite{minsker2017onsome}, Sec. 3.2] \label{thm:bernstein_operator}
Let $X_1,\dots,X_n$ be a sequence of independent Hermitian operators such that $\mathbb{E}[X_i] = 0$, for $i=1,\dots,n$ and $\sigma^2 \geq \|\sum_{i=1}^{n} \mathbb{E} X_i^2\|$
Assume that $\|X_i\| \leq U$ almost surely for all $1 \leq i \leq n$ and some positive $U \in \R$. Then, for any $t \geq \frac{1}{6}(U + \sqrt{U^2 + 36 \sigma^2})$,
\begin{align}
    \mathrm{Pr}\bigg(\bigg\|\sum_{i=1}^{n} X_i\bigg\| > t \bigg) \leq 14 r\bigg(\sum_{i=1}^{n} \mathbb{E} X_i^2 \bigg) \exp \bigg(-\frac{t^2/2}{\sigma^2 + t U/3} \bigg),
\end{align}
   where $r(\cdot)$ stands for the effective rank: $r(A) := \mathrm{tr}(A)/\|A\|$.
\end{lemma}

\textit{\textbf{Proof of \autoref{lem:empirical_RKRD}.}}
First, we see that eigenfunctions of $\left( (\hat{\Sigma}_q + \lambda \mathrm{Id})^{\frac{1-\alpha}{2\alpha}} \hat{\Sigma}_p (\hat{\Sigma}_q + \lambda \mathrm{Id})^{\frac{1-\alpha}{2\alpha}} \right)^\alpha$ must be of the form 
\begin{align} \label{eq:form_eigenfunctions}
f = \sum_{i=1}^{n} \alpha_i \phi(x_i) + \beta_i \phi(y_i).
\end{align}
To obtain this, note that the eigenfunctions of this operator are the same as those of $(\hat{\Sigma}_q + \lambda \mathrm{Id})^{\frac{1-\alpha}{2\alpha}} \hat{\Sigma}_p (\hat{\Sigma}_q + \lambda \mathrm{Id})^{\frac{1-\alpha}{2\alpha}}$, and that $\mathrm{Im}(\hat{\Sigma}_p) = \mathrm{span}((\phi(x_i))_{i=1}^{n})$. Using that $(\hat{\Sigma}_q + \lambda \mathrm{Id})^{\frac{1-\alpha}{2\alpha}}$ is equal to $\lambda^{\frac{1-\alpha}{2\alpha}} \mathrm{Id}$ when restricted to the orthogonal complement of $\mathrm{span}((\phi(y_i))_{i=1}^{n})$, and that the image of $\hat{\Sigma}_p$ restricted to $\mathrm{span}((\phi(y_i))_{i=1}^{n})$ is itself, we obtain that $\mathrm{Im}((\hat{\Sigma}_q + \lambda \mathrm{Id})^{\frac{1-\alpha}{2\alpha}} \hat{\Sigma}_p) = \mathrm{span}((\phi(x_i))_{i=1}^{n},(\phi(y_i))_{i=1}^{n})$, which shows equation \eqref{eq:form_eigenfunctions}.

Second, we obtain matrix expressions for the operators $\hat{\Sigma}_p$ and $\hat{\Sigma}_q + \lambda \mathrm{Id}$ on the space $\mathrm{span}((\phi(x_i))_{i=1}^{n},(\phi(y_i))_{i=1}^{n})$. If $f$ is given by \eqref{eq:form_eigenfunctions}, we have
\begin{align}
    \hat{\Sigma}_p f &= \bigg(\frac{1}{n} \sum_{i=1}^{n} \phi(x_i) \phi(x_i)^{*}\bigg) \bigg(\sum_{j=1}^{n} \alpha_j \phi(x_j) + \beta_j \phi(y_j)\bigg) \\ &= \sum_{i=1}^{n} \phi(x_i) \bigg(\sum_{j=1}^{n} \alpha_j \frac{1}{n}k(x_i,x_j) + \beta_j \frac{1}{n}k(x_i,y_j) \bigg),
\end{align}
If we write $\hat{\Sigma}_p f = \sum_{i=1}^{n} \alpha'_i \phi(x_i) + \beta'_i \phi(y_i)$, we obtain 
\begin{align}
    \begin{bmatrix}
    \alpha' \\
    \beta'
    \end{bmatrix} =
    \begin{bmatrix}
    \frac{1}{n}K_{x,x} & \frac{1}{n}K_{x,y} \\
    0 & 0
    \end{bmatrix}
    \begin{bmatrix}
    \alpha \\
    \beta
    \end{bmatrix} = \frac{1}{n}K_p \begin{bmatrix}
    \alpha \\
    \beta
    \end{bmatrix}.
\end{align}
Analogously, if we write $(\hat{\Sigma}_q + \lambda \mathrm{Id}) f = \sum_{i=1}^{n} \alpha'_i \phi(x_i) + \beta'_i \phi(y_i)$, we get
\begin{align}
    \begin{bmatrix}
    \alpha' \\
    \beta'
    \end{bmatrix} =
    \begin{bmatrix}
    \lambda \mathrm{Id} & 0 \\
    \frac{1}{n}K_{y,x} & \frac{1}{n}K_{y,y} + \lambda \mathrm{Id}
    \end{bmatrix}
    \begin{bmatrix}
    \alpha \\
    \beta
    \end{bmatrix} = \left(\frac{1}{n}K_q + \lambda \mathrm{Id} \right) \begin{bmatrix}
    \alpha \\
    \beta
    \end{bmatrix}.
\end{align}
In other words, the matrices for the linear operators $\hat{\Sigma}_p$ and $\hat{\Sigma}_q + \lambda \mathrm{Id}$ restricted to the finite-dimensional space $\mathrm{span}((\phi(x_i))_{i=1}^{n},(\phi(y_i))_{i=1}^{n})$ are $\frac{1}{n}K_p$ and $\frac{1}{n}K_q + \lambda \mathrm{Id}$, respectively. Consequently, the eigenvalues of $\hat{\Sigma}_p$ and $\hat{\Sigma}_q + \lambda \mathrm{Id}$ restricted to $\mathrm{span}((\phi(x_i))_{i=1}^{n},(\phi(y_i))_{i=1}^{n})$ are equal to the eigenvalues of $\frac{1}{n}K_p$ and $\frac{1}{n}K_q + \lambda \mathrm{Id}$, and their eigenvectors follow the correspondence $\sum_{i=1}^{n} \alpha_i \phi(x_i) + \beta_i \phi(y_i) \leftrightarrow [\alpha,\beta]^{\top}$. Hence, $(\hat{\Sigma}_q + \lambda \mathrm{Id})^{\frac{1-\alpha}{2\alpha}} f = (\frac{1}{n}K_q + \lambda \mathrm{Id})^{\frac{1-\alpha}{2\alpha}} [\alpha,\beta]^{\top}$. Since $(\hat{\Sigma}_q + \lambda \mathrm{Id})^{\frac{1-\alpha}{2\alpha}}$ and $\hat{\Sigma}_p$ map $\mathrm{span}((\phi(x_i))_{i=1}^{n},(\phi(y_i))_{i=1}^{n})$ to itself, the restriction of the composition of operators $(\hat{\Sigma}_q + \lambda \mathrm{Id})^{\frac{1-\alpha}{2\alpha}} \hat{\Sigma}_p (\hat{\Sigma}_q + \lambda \mathrm{Id})^{\frac{1-\alpha}{2\alpha}}$ admits the matrix expression $(\frac{1}{n}K_q + \lambda \mathrm{Id})^{\frac{1-\alpha}{2\alpha}} \frac{1}{n}K_p (\frac{1}{n}K_q + \lambda \mathrm{Id})^{\frac{1-\alpha}{2\alpha}}$.

That is, we get that $(\hat{\Sigma}_q + \lambda \mathrm{Id})^{\frac{1-\alpha}{2\alpha}} \hat{\Sigma}_p (\hat{\Sigma}_q + \lambda \mathrm{Id})^{\frac{1-\alpha}{2\alpha}} f = (\frac{1}{n}K_q + \lambda \mathrm{Id})^{\frac{1-\alpha}{2\alpha}} \frac{1}{n}K_p (\frac{1}{n}K_q + \lambda \mathrm{Id})^{\frac{1-\alpha}{2\alpha}} [\alpha,\beta]^{\top}$. By the same argument, the eigenvalues of the restriction of this operator are equal to the eigenvalues of the corresponding matrix, and so is the sum of the eigenvalues raised to the power $\alpha$, which concludes the proof.
\qed

\section{Proofs of \autoref{sec:auditing}} \label{sec:proofs_auditing}
\begin{lemma} \label{lem:hypothesis_test_1}
Let $p, q \in \mathcal{P}(\mathcal{R})$ be such that for all $S \subseteq \mathcal{R}$, $p(S) \leq e^{\varepsilon} q(S) + \frac{\delta}{2}$, for some $\varepsilon, \delta > 0$. Let $\Sigma_p = \int_{\mathcal{R}} \phi(x) \phi(x)^{*} \, dp(x)$, $\Sigma_q = \int_{\mathcal{R}} \phi(x) \phi(x)^{*} \, dq(x)$. We have that
\begin{align}
    \Sigma_p \leq e^{\varepsilon} \Sigma_q + \delta \mathrm{Id}.
\end{align}
\end{lemma}
\begin{proof}
We show the contrapositive. Let $p = \mathcal{L}(f(D))$, $q = \mathcal{L}(f(D'))$ be the laws of $f(D), f(D')$ for adjacent $D, D'$. Suppose that $\Sigma_p \not\leq e^{\varepsilon} \Sigma_q + \delta \mathrm{Id}$. Then, there exists $f \in \mathcal{H}$ such that
\begin{align}
    \langle f, \Sigma_p f \rangle > e^{\varepsilon} \langle f, \Sigma_q f \rangle + \delta \|f\|^2,
\end{align}
which we can rewrite as
\begin{align}
    \mathbb{E}_{x \sim p} [\langle f, \phi(x) \rangle^2] - e^{\varepsilon} \mathbb{E}_{x \sim q} [\langle f, \phi(x) \rangle^2] - \delta \|f\|^2 > 0.
\end{align}
In the following, we assume without loss of generality that $\|f\|^2 = 1$, and then this simplifies to $\mathbb{E}_{x \sim p} [\langle f, \phi(x) \rangle^2] - e^{\varepsilon} \mathbb{E}_{x \sim q} [\langle f, \phi(x) \rangle^2] - \delta > 0$. We define the map $T_f : \mathcal{R} \to \mathbb{R}$ as $T_f(x) = \langle f, \phi(x) \rangle^2$. Note that $0 \leq \langle f, \phi(x) \rangle^2 \leq \|f\|^2 \|\phi(x)\|^2 = 1$. We can rewrite the inequality as
\begin{align} \label{eq:ineq_rewritten}
    \int_{0}^{1} x \, d(T_f)_{\#}p(x) - e^{\varepsilon} \int_{0}^{1} x \, d(T_f)_{\#}q(x) - \delta > 0.
\end{align}
Since
\begin{align}
    \int_{0}^{1} x \, d(T_f)_{\#}p(x) &= \left[s \mathrm{Prob}_{x \sim p}[T_f(x) \leq s] \right]_{0}^{1}-\int_{0}^{1} \mathrm{Prob}_{x \sim p}[T_f(x) \leq s] \, ds \\ &= \int_{0}^{1} (1-\mathrm{Prob}_{x \sim p}[T_f(x) \leq s]) \, ds = \int_{0}^{1} \mathrm{Prob}_{x \sim p}[T_f(x) \geq s] \, ds,
\end{align}
we can rewrite \eqref{eq:ineq_rewritten} as
\begin{align}
    \int_{0}^{1} \left( \mathrm{Prob}_{x \sim p}[T_f(x) \geq s] - e^{\varepsilon} \mathrm{Prob}_{x \sim q}[T_f(x) \geq s] - \delta \right) \, ds > 0
\end{align}
This implies that for some $s \in [-1,1]$, we must have that $\mathrm{Prob}_{x \sim p}[T_f(x) \geq s] - e^{\varepsilon} \mathrm{Prob}_{x \sim q}[T_f(x) \geq s] - \delta > 0$. Hence, if we define the event $S = \{ x | T_f(x) \geq s \}$, we obtain that $p(S) > e^{\varepsilon} q(S) + \delta$, which concludes the proof. 
\end{proof}

\textit{\textbf{Proof of \autoref{thm:epsilon_delta_implies}.}}
For any adjacent $D, D' \in \mathcal{D}$, let $p = \mathcal{L}(f(D))$ and $q = \mathcal{L}(f(D'))$ be the laws of $f(D)$, $f(D')$. Applying \autoref{lem:hypothesis_test_1}, we obtain that $\Sigma_p \leq e^{\varepsilon} \Sigma_q + \delta \mathrm{Id} = e^{\varepsilon} (\Sigma_q + \delta e^{-\varepsilon} \mathrm{Id})$. That is $\Sigma_q + \delta e^{-\varepsilon} \mathrm{Id} \geq e^{-\varepsilon} \Sigma_p$. Applying \autoref{prop:4_lennert}, we get
\begin{align}
    D_{\alpha,\delta e^{-\varepsilon}}(\Sigma_p||\Sigma_q) &= D_{\alpha}(\Sigma_p||\Sigma_q + \delta e^{-\varepsilon} \mathrm{Id}) \leq D_{\alpha}(\Sigma_p||e^{-\varepsilon} \Sigma_p) \\ &= \frac{1}{\alpha-1} \log \left( \mathrm{tr}\left[\left( (e^{-\varepsilon} \Sigma_p)^{\frac{1-\alpha}{2\alpha}} \Sigma_p (e^{-\varepsilon} \Sigma_p)^{\frac{1-\alpha}{2\alpha}} \right)^\alpha \right] \right) = \frac{1}{\alpha-1} \log ( e^{-(1-\alpha)\varepsilon}) = \varepsilon.
\end{align}
\qed

\section{Experimental details and potential impact} \label{sec:exp_details}
The code, which can be found at \url{https://github.com/CDEnrich/kernel_renyi_dp}, is in Python, using Jupyter Notebook. It was run on a personal laptop. The computations for the plots in \autoref{fig:rkrd_audit} took about 90 minutes. To compute the RKRD, we used the radial basis function (RBF) kernel with the median criterion for the bandwidth.

We do not foresee negative societal impact of our work as its aim is to prevent privacy leakage in applied differential privacy.

\end{document}